\documentclass[11pt]{article}
\usepackage{amsmath,amssymb,amsthm,mathtools}
\usepackage{geometry}
\usepackage[ruled,vlined]{algorithm2e}
\usepackage{xcolor}
\usepackage{booktabs}
\usepackage{graphicx}
\usepackage{subcaption}
\usepackage[hidelinks]{hyperref}
\usepackage{float}
\usepackage[section]{placeins}
\usepackage{tikz}
\usetikzlibrary{arrows.meta,positioning}

\graphicspath{{./}{figures/}}
\theoremstyle{plain}
\newtheorem{theorem}{Theorem}[section]
\newtheorem{proposition}[theorem]{Proposition}
\newtheorem{lemma}[theorem]{Lemma}
\newtheorem{corollary}[theorem]{Corollary}
\theoremstyle{definition}

\theoremstyle{remark}
\newtheorem{remark}[theorem]{Remark}
\newcommand\keyword[1]{\textbf{Keywords}: #1}

\newcommand{\E}{\mathbb{E}}
\newcommand{\R}{\mathbb{R}}
\newcommand{\Sph}{\mathbb{S}}
\newcommand{\Beta}{\mathrm{Beta}}
\newcommand{\GammaDist}{\mathrm{Gamma}}
\newcommand{\Linf}{L^\infty}
\newcommand{\essinf}{\operatorname*{ess\,inf}}
\newcommand{\ft}{\tilde f}
\newcommand{\BUC}{\mathrm{BUC}}
\newcommand{\Ecal}{\mathcal E}
\newcommand{\dd}{\mathrm{d}}

\newif\ifshowrevnotes
\showrevnotesfalse

\newcommand{\authorcheck}[1]{\ifshowrevnotes\par\smallskip\noindent
  \fcolorbox{blue!60!black}{blue!5}{\parbox{\dimexpr\linewidth-2\fboxsep-2\fboxrule\relax}{\small
  \textbf{\textcolor{blue!60!black}{AUTHOR CHECK:}} #1}}\par\smallskip\fi}

\title{A Unified Heat-Ball Framework for Fokker--Planck Equations:\\
Mean-Value Identities, Exact Sampling, and Error Analysis}

\begin{document}

\title{A Unified Heat-Ball Framework for Fokker--Planck Equations:\\
Mean-Value Identities, Exact Sampling, and Error Analysis}

\author{
    Jiayao Bai\textsuperscript{1}, \quad 
    Yi Du\textsuperscript{1,}\textsuperscript{*}, \quad 
    Rongxing Lai\textsuperscript{1}, \quad 
    Yifei Jia\textsuperscript{2}
}
\date{}

\maketitle

% --- 手动生成一个完美掌控的“脚注” ---
\begingroup
\renewcommand\thefootnote{}\footnotetext{
    \textsuperscript{*}Corresponding author. Emails: jnbjy@stu2025.jnu.edu.cn (Jiayao Bai), duyidy@jnu.edu.cn (Yi Du), 2960631875@qq.com (Rongxing Lai), jiayifei333@163.com (Yifei Jia).
}
\endgroup
% -------------------------------------

\vspace{-0.5cm} 
\noindent
\begin{center}
\textsuperscript{1}Department of Mathematics, College of Information Science and Technology, Jinan University, Guangzhou, Guangdong, China \\
\textsuperscript{2}College of Mathematics and System Sciences, Xinjiang University, Urumqi, Xinjiang, China
\end{center}
\vspace{0.5cm}
%==========================================================
\begin{abstract}
Diffusion-based generative models with nonlinear or learned forward drifts
raise questions that stay hidden for Gaussian forward processes: how the
divergence-free and gradient parts of the drift shape the evolving density,
whether sampling can draw on local rather than global information, and how
to treat low-density regions, where $u$ carries a weak signal while
$\log u$ diverges and the entropy density $u\log u$ falls between the two.
These questions motivate the present study of local representations for
Fokker--Planck equations with general, possibly nonlinear and
time-dependent drift.

We derive inhomogeneous heat-ball identities for the density $u$, the
log-density $\log u$, the entropy density $u\log u$, and the score
$\nabla\log u$, keeping the source terms generated by the drift; their
unscaled residuals recover the corresponding differential residuals at
leading order in the squared heat-ball radius. The weights appearing in
the identities, the source integrals, and the integration-by-parts formulas
all belong to a single Beta--Gamma family of measures on the heat ball,
which gives exact rejection-free sampling, product quadrature, and unbiased
Monte Carlo estimators with variance and cost bounds that are explicit in
the dimension and the radius. Experiments with Gaussian-mixture and
nonlinear-drift references, in dimensions up to $1024$ where applicable,
confirm the sampling laws, the identities, and the residual behavior. In
particular, log-density and score evaluations remain accurate where direct
density estimation breaks down, and the heat-ball residuals provide local
consistency tests for learned score models. A study of gauge-based
heat-ball walks illustrates a recursive use of the density identity. The
framework supplies representation and numerical-integration tools for the
three motivating questions; their full consequences for generative modeling
are left to subsequent work.
\end{abstract}

\keyword
Fokker-Planck equation, Heat-ball identity, Score-based models, Exact sampling, Derivative-free residuals



%==========================================================
\section{Introduction}\label{sec:intro}
%==========================================================

Score-based generative models construct samples by reversing a diffusion
that transports the data distribution toward a reference law. The reverse
dynamics depends on the score of the time-dependent marginal density.
For the linear Gaussian forward processes used in DDPM and related models,
conditional transition densities and scores are explicit, and denoising
score matching provides a practical training objective
\cite{Ho2020,Song2019,Song2020,Vincent2011}. This convenient structure is
lost for many nonlinear or learned forward drifts. Approximate conditional
scores, local Gaussian methods, and sliced score
matching~\cite{SongGargShi2020} offer alternatives, but
they do not by themselves describe how different components of a general
drift affect the evolving marginal density and its score. This is one
reason to study the Fokker--Planck equation beyond the Gaussian setting.

A particularly useful distinction is between the divergence-free and
gradient parts of the drift. Under suitable conditions and with an
appropriate choice of domain and boundary conditions, one may write
$f=f_{\mathrm{df}}+\nabla\phi$, where
$\nabla\cdot f_{\mathrm{df}}=0$. The two parts enter the probability
transport differently. A divergence-free field need not leave the density
unchanged: its advection can still rearrange a nonuniform distribution,
while the gradient part may create confining or expansive motion. Their
combined effect on marginal scores, especially for nonlinear and
time-dependent drifts, is not transparent from a local Gaussian
approximation. We therefore retain the full drift in the Fokker--Planck
equation instead of restricting the analysis to an explicitly solvable
forward process. The decomposition serves here as physical and
computational motivation; we do not attempt a general characterization
of what each part contributes to generative performance.

Meanwhile,  standard score estimators are
functions of the entire high-dimensional state, even when the data have
strong local interactions. For a spatially structured signal, it is
natural to ask how much information can be extracted from a neighborhood
of a state rather than from a global transition formula. Heat balls
give a precise meaning to locality for parabolic equations: their
mean-value identities involve only values within a bounded neighborhood in
state space and backward time. This does not by itself make a pixel or a
physical degree of freedom conditionally independent of distant ones, but
it supplies local integral relations that can be used alongside models
whose drift reflects the spatial structure of the data.

Furthermore,  direct
density values become small near vacuum, whereas the log-density diverges:
$\log u\to-\infty$ as $u\to0^+$. If unregularized log-density values
are used to weight a numerical objective, extremely small densities may
receive disproportionate influence. The entropy density $w=u\log u$
provides an intermediate pointwise scale. For sufficiently small $u>0$,
\begin{equation}\label{eq:intro-entropy-scale}
 |u\log u|=u|\log u|,
 \qquad \frac{|u\log u|}{u}=|\log u|\longrightarrow\infty,
 \qquad u\log u\longrightarrow0.
\end{equation}
Thus its magnitude is enhanced relative to $u$ by a logarithmic factor,
while the function value remains finite and vanishes at vacuum. This
motivates studying the entropy density together with $u$, $\log u$, and
$\nabla\log u$. The comparison is one of pointwise scales; on its own it
says nothing about bounded gradients or improved rare-mode generation
under an entropy-weighted objective.

These considerations lead to the Fokker--Planck equation
\begin{equation*}
 \partial_t p=-\nabla\cdot(fp)+\tfrac12g^2(t)\Delta p,
\end{equation*}
for a possibly nonlinear, time-dependent drift $f$. We seek local
relations for the density, log-density, entropy density, and score that
retain the effects of the full drift. A standard cumulative-variance
time change rewrites this equation as a heat equation with a source.
The transformation itself is elementary; the difficulty is that the
source depends on the unknown solution, and the equations for the
log-density, entropy density, and score contain additional nonlinear
terms. Watson's heat-ball formula gives a starting point for deriving
local identities without discarding those terms
\cite{watson1973,Evans2010}. Unlike a closed transition kernel, the
resulting identities are consistency and evaluation relations: using
them numerically still requires information about the unknown functions
on the heat ball or a suitable approximation to those functions.

A local identity is useful computationally only if its integrals can be
evaluated. The relevant weights are not all the same: the Watson average,
the source integral, and integration-by-parts formulas lead to different
measures on the heat ball. In high dimensions, naive sampling can be
costly or unstable. This raises the numerical question that drives the
main theoretical construction of the paper: whether these integrals
share a tractable sampling structure, and what error and computational
cost are incurred when they are estimated. Earlier heat-ball and walk
methods primarily treated heat equations or equations with prescribed
sources \cite{DeaconuHerrmannMaire2017,DeaconuHerrmann2018,Sabelfeld2017,Sawhney2020,GridFreeDiffusion2026}.
Bao et al.~\cite{Bao2026} also use Gamma-distributed logarithmic times in
source integration, so the Gamma-time variable in isolation is not the
novelty sought here.

We prove exact heat-ball identities for $u$, $v=\log u$, $w=u\log u$,
and $s=\nabla\log u$, with the corresponding drift-dependent sources.
Their residuals have a small-radius expansion related to the differential
Fokker--Planck residual. We then identify a family of heat-ball measures
with weights proportional to $|z|^{2a}\psi_0^b\sigma^{-c}$; after
normalization, each measure factors into an independent spherical
direction, Beta radial variable, and Gamma logarithmic-time variable.
This common structure yields rejection-free samplers, product quadrature,
and estimates for sampling variance and cost. It also makes possible
local residual diagnostics for approximations of the score. Numerical
experiments test the identities and estimators under linear, rotating,
and nonlinear drift; a brief application to heat-ball walks illustrates
another way to use the density identity.

The contributions of this paper are analytical and numerical: a family of
local identities for the quantities that matter in general-drift
generative modeling, a unified way of evaluating the resulting heat-ball
integrals, and error estimates and experiments for the estimators. The
three generative-model questions above motivate the construction. We do
not settle them here---a full analysis of drift decomposition in reverse
generation, a proof of spatially local dependence in a generative
network, and an empirical demonstration that entropy weighting improves
low-density samples all remain open---but the identities and tools
developed below are what is needed to attack them directly.

Section~\ref{sec:model} describes the equation and regularity assumptions.
Sections~\ref{sec:heatball}--\ref{sec:error} establish the identities,
sampling method, and error analysis. Sections~\ref{sec:score-models} and
\ref{sec:experiments} discuss score residuals and numerical validation.
Section~\ref{sec:walks} briefly considers recursive heat-ball walks.
Figure~\ref{fig:framework} summarizes the organization.

\begin{figure}[ht]
\centering
\begin{tikzpicture}[
  box/.style={draw, rounded corners=2pt, align=center, font=\small, inner sep=3pt, minimum height=17mm, text width=3.15cm},
  main/.style={box, fill=blue!5},
  app/.style={box, dashed, minimum height=11mm},
  arr/.style={-{Stealth[length=2mm]}, semithick},
  lab/.style={font=\scriptsize, align=center, text=black!65}]
\node[main] (fp) {Fokker--Planck equation $\to$ heat equation with source (Section~\ref{sec:model})};
\node[main, right=9mm of fp] (id) {heat-ball identities for $u$, $\log u$, $u\log u$, $\nabla\log u$ (Section~\ref{sec:heatball})};
\node[main, right=9mm of id] (fam) {Beta--Gamma family: exact sampling or Gauss rules (Section~\ref{sec:sampling})};
\node[main, right=9mm of fam] (err) {unbiasedness, variance and cost in $d$, $r$ (Section~\ref{sec:error})};
\draw[arr] (fp) -- (id);
\draw[arr] (id) -- (fam);
\draw[arr] (fam) -- (err);
\node[app, below=10mm of err] (res) {residuals for learned score models (Section~\ref{sec:score-models})};
\node[app, below=10mm of id] (walk) {heat-ball walks for the density (Section~\ref{sec:walks})};
\draw[arr] (err) -- (res);
\draw[arr] (fam.south) |- (res.west);
\draw[arr] (id) -- node[lab, left] {density\\identity} (walk);
\node[lab, font=\footnotesize\itshape, below=12mm of fp, text width=3.15cm] {each step tested against exact references (Section~\ref{sec:experiments})};
\end{tikzpicture}
\caption{Structure of the framework. Solid boxes: the main line of the paper
(Sections~\ref{sec:model}--\ref{sec:error}); dashed boxes: applications.}
\label{fig:framework}
\end{figure}

\paragraph{Other related work.}
Derivative-free losses based on stochastic representations have been used to
train neural PDE solvers~\cite{HanNicaStinchcombe2020}; their squared Monte
Carlo targets are biased by the walker variance, which we remove with a
two-sample estimator (Section~\ref{sec:ustat}). The pointwise Fokker--Planck
residual of a learned score is used as a regularizer in
FP-Diffusion~\cite{Lai2023}, and learning the score of a Fokker--Planck
solution along probability-flow trajectories is studied in~\cite{Boffi2023}.
Error bounds relating the score error to the error of the generated law are
available for linear forward processes~\cite{Chen2023}. Training with the
residuals of this paper, and an entropy-tilted training law for low-density
regions, are studied in a companion paper~\cite{CompanionML}.

%==========================================================
\section{Problem setup and well-posedness}\label{sec:model}
%==========================================================

Consider the It\^o SDE
\begin{equation}\label{eq:gen_sde}
  \dd x_t = f(x_t,t)\,\dd t + g(t)\,\dd W_t,
  \qquad x_t\in\R^d,\quad g(t)>0,
\end{equation}
whose density $p(x,t)$ solves the Fokker--Planck equation
$\partial_t p = -\nabla_x\cdot(f\,p) + \tfrac12 g^2(t)\,\Delta_x p$.
With the cumulative-variance coordinate
\begin{equation}\label{eq:tau}
  \tau(t) := \int_0^t \tfrac12 g^2(s)\,\dd s,
\end{equation}
$u(x,\tau):=p(x,t(\tau))$ and the normalized drift
$\ft(x,\tau):=f(x,t(\tau))/\bigl(\tfrac12g^2(t(\tau))\bigr)$, the
Fokker--Planck equation becomes the inhomogeneous heat equation
\begin{equation}\label{eq:heat}
  \partial_\tau u - \Delta_x u = -\nabla_x\cdot\bigl(\ft(x,\tau)\,u(x,\tau)\bigr).
\end{equation}
We work in the $(x,\tau)$ coordinates throughout, so that the diffusion
coefficient is one. No decomposition of $\ft$ is needed for the identities
or residuals in Sections~\ref{sec:heatball}--\ref{sec:score-models}.
A Helmholtz decomposition is used only for the walk construction in
Section~\ref{sec:walks}.

\paragraph{Well-posedness.}\label{sec:wellposed}
The heat-ball identities require a positive $C^{2,1}$ solution on the
heat ball under consideration. We give a global result for bounded drifts,
followed by a Lyapunov condition that supplies the interior properties
needed for the unbounded confining drifts used in the experiments. The
second result does not provide global Gaussian bounds.

\begin{theorem}[Well-posedness, positivity and local control]\label{thm:local}
Let $u_0\in L^1(\R^d)\cap\BUC^1(\R^d)$ be a probability density, where
$\BUC^1$ denotes the bounded, uniformly continuous functions whose first
derivatives are also bounded and uniformly continuous; set
$M_0:=\|u_0\|_{\Linf}$. Let the normalized drift satisfy
\[
\ft\in L^\infty_{\mathrm{loc}}\bigl([0,\infty);C^1_b(\R^d;\R^d)\bigr),
\]
and, for $T>0$, set
\[
A(T):=\sup_{0\le\tau\le T}\|\ft(\cdot,\tau)\|_{\Linf},\qquad
C(T):=\sup_{0\le\tau\le T}\|\nabla\!\cdot\!\ft(\cdot,\tau)\|_{\Linf}.
\]
Then~\eqref{eq:heat} has a unique global mild solution
$u\in C([0,\infty);\BUC^1(\R^d))$, which is a probability density for every
$\tau\ge0$, and:
\begin{enumerate}
  \item[(a)] \textbf{Regularity.} If $\ft$ and $\nabla\!\cdot\!\ft$ are
  locally H\"older continuous in $(x,\tau)$, then
  $u\in C^{2,1}_{\mathrm{loc}}(\R^d\times(0,\infty))$ is a classical
  solution; if $\ft$ is smooth with bounded derivatives, $u$ is smooth for
  $\tau>0$.
  \item[(b)] \textbf{Two-sided bounds.} $u(x,\tau)\le M_0e^{C(\tau)\tau}$.
  For every $T>0$ there are $c_T,\gamma_T>0$, depending only on $d$, $T$ and
  $A(T)$, such that for $0<\tau\le T$
  \begin{equation}\label{eq:aronson-lower}
    u(x,\tau)\ \ge\ c_T\,\tau^{-d/2}\!\int_{\R^d}e^{-\gamma_T|x-y|^2/\tau}u_0(y)\,\dd y
    \ \ge\ c_T\,\tau^{-d/2}\,m(\tau)\,e^{-2\gamma_T|x|^2/\tau},
  \end{equation}
  with $m(\tau):=\int e^{-2\gamma_T|y|^2/\tau}u_0(y)\,\dd y>0$.
\end{enumerate}
Consequently $v=\log u$ is well defined for $\tau>0$ and grows at most
quadratically,
$\log\bigl(c_T\tau^{-d/2}m(\tau)\bigr)-2\gamma_T|x|^2/\tau\le v(x,\tau)
\le\log M_0+C(\tau)\tau$, so $v$, $u\log u$ and $\nabla v$ are bounded on
compact subsets of $\R^d\times(0,\infty)$.
\end{theorem}

\begin{proof}[Proof sketch]
Full details are in Appendix~\ref{app:proof}. In nondivergence form
$\partial_\tau u=\Delta u-N[u]$ with
$N[u]=\ft\cdot\nabla u+(\nabla\cdot\ft)u$. On $C([a,b];\BUC^1)$ the Duhamel
map is a $\theta$-contraction whenever
$K\sqrt{b-a}\,(\sqrt{b-a}+2C_d)\le\theta$, where $K=A+C$ and
$C_d=\Gamma(\frac{d+1}2)/\Gamma(\frac d2)$, because
$\|\nabla e^{t\Delta}g\|_{\Linf}\le C_dt^{-1/2}\|g\|_{\Linf}$. The step
length does not depend on the data because the equation is linear, so the
step can be iterated and the solution is global. The upper bound follows by
comparison with $M_0e^{C\tau}$; the lower bound is Aronson's Gaussian lower
bound~\cite{Aronson1968} for the fundamental solution of
$\partial_\tau u=\partial_j(a_{ij}\partial_iu+a_ju)+b_j\partial_ju+cu$ with bounded measurable
coefficients and uniformly elliptic $(a_{ij})$. Here $a_{ij}=\delta_{ij}$, $a_j=-\ft_j$ and
$b=c=0$, and the constants of the bound depend only on $d$, $T$, the ellipticity
constants (here equal to one) and the bounds of the lower-order coefficients,
that is, on $d$, $T$ and $A(T)$.
\end{proof}

\begin{proposition}[Positive interior solution for smooth dissipative drifts]\label{prop:lyapunov}
Let $u_0\ge0$ be a probability density in $L^1(\R^d)$, not identically
zero. Suppose $\ft(x,\tau)$ is continuous in time, smooth in $x$, locally
Lipschitz in $x$ uniformly on compact time intervals, and its spatial
derivatives are locally bounded. Assume that, for each $T>0$, a constant
$K_T$ exists such that
\begin{equation}\label{eq:lyapunov}
  2x\cdot\ft(x,\tau)+2d\le K_T(1+|x|^2),
  \qquad x\in\R^d,\quad 0\le\tau\le T.
\end{equation}
Then $\mathrm dX_\tau=\ft(X_\tau,\tau)\,\mathrm d\tau+\sqrt{2}\,\mathrm dW_\tau$
is nonexplosive. Its transition law, and hence the law started from
$u_0$, has a strictly positive density for positive times. If in addition the time dependence is locally H\"older and the smooth
coefficients possess locally H\"older spatial derivatives, standard
interior parabolic regularity gives a $C^{2,1}_{\rm loc}$ density
for $\tau>0$. These conditions hold for smooth (also polynomially growing)
confining drifts. With smooth space--time coefficients the density is
smooth in the interior. Consequently the
local identities in Section~\ref{sec:heatball} hold on every heat ball
whose closure lies in $\tau>0$. No globally uniform Aronson bounds or
$\BUC^1$ estimates are asserted for this unbounded-drift class.
\end{proposition}
\begin{proof}
For $V(x)=1+|x|^2$, the generator $L_\tau=\Delta+\ft\cdot\nabla$
satisfies $L_\tau V\le K_T V$ by~\eqref{eq:lyapunov}. Apply It\^o's
formula up to the first exit from $B_R$ and Gronwall's inequality:
$\E V(X_{\tau\wedge\tau_R})\le e^{K_T\tau}V(X_0)$ for deterministic
initial points. Since $V\ge1+R^2$ on $\{\tau_R\le T\}$, the
probability of exit before $T$ tends to zero as $R\to\infty$.
On each bounded space--time cylinder the diffusion is uniformly elliptic
with smooth, bounded local coefficients. Local parabolic fundamental-solution
regularity and the strong maximum principle give a positive interior
transition density; positivity propagates from any nonzero nonnegative
initial distribution. Interior parabolic regularity then gives the stated
regularity. These are local conclusions and do not use the global Gaussian
estimate in Theorem~\ref{thm:local}(b).
\end{proof}
\begin{remark}[The nonlinear double-well example]\label{rem:doublewell-scope}
For $\ft(x)=-(|x|^2/4-1)x+\omega Jx/(1+|x|^2/4)$ with $J^\top=-J$,
$2x\cdot\ft(x)=-|x|^4/2+2|x|^2$, so~\eqref{eq:lyapunov}
holds. Proposition~\ref{prop:lyapunov}, rather than the bounded-drift
part of Theorem~\ref{thm:local}, justifies the interior identities in
this example. The positivity conclusion alone does not imply a global
quadratic bound on $\log u$ for this nonlinear drift.
\end{remark}

\begin{remark}\label{rem:local}
No assumption of the form $\essinf_{\R^d}u_0>0$ is made: an integrable
density cannot satisfy it. The identities of Section~\ref{sec:heatball}
need only positivity and $C^{2,1}$ regularity on each heat ball, supplied
by Theorem~\ref{thm:local} for bounded drifts and by
Proposition~\ref{prop:lyapunov} for the smooth dissipative examples. The step length of the
contraction scales like $d^{-1}$ through $C_d\sim\sqrt{d/2}$; this
reflects the norm used in the estimate and has no consequence for
existence, which is global in every dimension.
\end{remark}

%==========================================================
\section{The heat-ball mean-value hierarchy}\label{sec:heatball}
%==========================================================

Applying Watson's formula to~\eqref{eq:heat} gives a local identity for
$u$ with the drift retained as a source. The equations for $\log u$,
$u\log u$, and $\nabla\log u$ have their own, generally nonlinear
sources. We derive the corresponding identities and their small-radius
residual expansion here; Section~\ref{sec:sampling} treats their numerical
integration.

\subsection{Heat balls}
For a reference point $(x_0,\tau_0)$ let
\begin{equation}\label{eq:Phi}
  \Phi(x-x_0,\tau_0-\tau)
  := \bigl(4\pi(\tau_0-\tau)\bigr)^{-d/2}
     \exp\!\Bigl(-\frac{|x-x_0|^2}{4(\tau_0-\tau)}\Bigr),\qquad \tau<\tau_0,
\end{equation}
be the backward heat kernel, and for $r>0$ define the heat ball
$\Ecal(x_0,\tau_0;r):=\{(x,\tau):\tau<\tau_0,\ \Phi(x-x_0,\tau_0-\tau)\ge r^{-d}\}$
and the potential $\psi:=\log\bigl(r^d\Phi\bigr)$, which is positive inside
and vanishes on the boundary. The parabolic scaling
$(z,\sigma)=((x-x_0)/r,(\tau_0-\tau)/r^2)$ maps $\Ecal(x_0,\tau_0;r)$
onto the unit heat ball $\Ecal_1=\{(z,\sigma):\sigma>0,\ \Phi(z,\sigma)\ge1\}$
and $\psi$ onto
\begin{equation}\label{eq:psi0}
  \psi_0(z,\sigma)=-\tfrac d2\log(4\pi\sigma)-\frac{|z|^2}{4\sigma}
  =\frac{R(\sigma)^2-|z|^2}{4\sigma},\qquad
  R(\sigma)^2:=-2d\,\sigma\log(4\pi\sigma),
\end{equation}
so that $\Ecal_1=\{(z,\sigma):0<\sigma<1/(4\pi),\ |z|\le R(\sigma)\}$.
Since $\sigma<1/(4\pi)$ on $\Ecal_1$, the heat ball of radius $r$ lies in
$\{\tau>0\}$ if and only if $r^2 < 4\pi\tau_0$. Throughout, $(z,\sigma)$
denote rescaled space and backward time; the letter $s$ is reserved for the
score.

\subsection{Mean-value identities}
For $F$ defined near the heat ball, let
\begin{equation}\label{eq:Mr}
  M_r[F](x_0,\tau_0)
  :=\frac14\iint_{\Ecal_1}F(x_0+rz,\tau_0-r^2\sigma)\,
    \frac{|z|^2}{\sigma^2}\,\dd z\,\dd\sigma,
\end{equation}
and for a continuous $h$ let
\begin{equation}\label{eq:Sr}
  S_r[h](x_0,\tau_0):=d\iint_{\Ecal_1}\psi_0(z,\sigma)
    \int_0^r\rho\,h(x_0+\rho z,\tau_0-\rho^2\sigma)\,\dd\rho\,\dd z\,\dd\sigma.
\end{equation}

\begin{lemma}[Inhomogeneous Watson identity]\label{lem:monotone}
If $F\in C^{2,1}$ on a neighbourhood of the closure of
$\Ecal(x_0,\tau_0;r)$ and $\partial_\tau F-\Delta F=h$ there, then
$F(x_0,\tau_0)=M_r[F](x_0,\tau_0)+S_r[h](x_0,\tau_0)$.
\end{lemma}

\begin{proof}
For $h=0$ this is Watson's identity~\cite{watson1973},
\cite[\S2.3.2]{Evans2010}. In general, Evans' monotonicity argument applied
to $J(r)=r^{-d}\iint_{\Ecal(r)}F\,|x-x_0|^2/(\tau_0-\tau)^2
+\int_0^r 4d\rho^{-d-1}\iint_{\Ecal(\rho)}h\,\psi$ gives $J'(r)=0$ and
$J(r)\to4F(x_0,\tau_0)$ as $r\to0$ (Appendix~\ref{app:proof-lemmas}); the
parabolic rescaling turns $J(r)=4F(x_0,\tau_0)$ into the stated form.
\end{proof}

Writing $v=\log u$ and $w=u\log u$, direct computation from~\eqref{eq:heat}
gives the sources in Table~\ref{tab:sources}.

\begin{table}[ht]
\centering
\caption{Sources $h_F$ in $\partial_\tau F-\Delta F=h_F$ for the quantities
derived from a positive solution $u$ of~\eqref{eq:heat}.}
\label{tab:sources}
\begin{tabular}{ll}
\toprule
$F$ & $h_F$ \\
\midrule
$u$ & $-\nabla\cdot(\ft u)$ \\
$v=\log u$ & $|\nabla v|^2-\ft\cdot\nabla v-\nabla\cdot\ft$ \\
$w=u\log u$ & $-\nabla\cdot(\ft w)-u|\nabla v|^2-u\,\nabla\cdot\ft$ \\
$s=\nabla v$ & $\nabla h_v$ \\
\bottomrule
\end{tabular}
\end{table}

\begin{theorem}[Mean-value hierarchy]\label{thm:hierarchy}
Suppose $u>0$ and $u\in C^{2,1}_{\mathrm{loc}}(\R^d\times(0,\infty))$
solves~\eqref{eq:heat}. Then for every $\tau_0>0$,
$0<r<\sqrt{4\pi\tau_0}$ and $F\in\{u,v,w\}$,
\begin{equation}\label{eq:mv-general}
  F(x_0,\tau_0)=M_r[F](x_0,\tau_0)+S_r[h_F](x_0,\tau_0).
\end{equation}
\end{theorem}

\begin{proof}
By the assumptions, $u$ is $C^{2,1}$ and positive for $\tau>0$, so
$v=\log u$ and $w=uv$ are $C^{2,1}$ on the closed heat ball. The sources
follow by substituting $u=e^v$ and $w=uv$ into~\eqref{eq:heat}; apply
Lemma~\ref{lem:monotone}.
\end{proof}

\begin{corollary}[Score identity]\label{cor:score-mean}
If in addition $h_v\in C^1$, then componentwise, with $\nabla s$ the Jacobian
$(\partial_js_i)_{ij}$ and $\nabla\ft$ likewise,
\begin{equation}\label{eq:score-mean}
  s(x_0,\tau_0)=M_r[s](x_0,\tau_0)+S_r[\nabla h_v](x_0,\tau_0),\qquad
  \nabla h_v=(\nabla s)^{\!\top}(2s-\ft)-(\nabla\ft)^{\!\top}s-\nabla(\nabla\cdot\ft).
\end{equation}
\end{corollary}

\begin{proof}
In rescaled variables the domain and weights do not depend on $x_0$, so
differentiating~\eqref{eq:mv-general} for $v$ under the integral sign
gives~\eqref{eq:score-mean}. The formula for $\nabla h_v$ is the chain rule.
\end{proof}

\begin{remark}\label{rem:implicit}
The source terms for $v$, $w$, and $s$ contain the unknown solution, so
these formulas are not closed representations from the drift and initial
data alone. They evaluate a quantity from its values on a neighboring
heat ball or test the consistency of an approximate solution. The entropy
density has a particular low-density scaling: although $u\log u\to0$ at
vacuum, $|u\log u|/u=|\log u|\to\infty$, so it keeps a larger relative
signal than $u$ without the divergent function values of $\log u$. This
scaling is what motivates the entropy identity; whether it reduces Monte
Carlo variance or improves score learning is a separate question.
\end{remark}

\subsection{Small-radius limit}
For $F\in C^{2,1}$ and a continuous candidate source $h$, define the
heat-ball residual $\mathcal R_r[F;h]:=F-M_r[F]-S_r[h]$.

\begin{lemma}[Taylor expansion]\label{lem:expansion}
Let $C_1:=\iint_{\Ecal_1}|z|^2/\sigma\,\dd z\,\dd\sigma$. Then
$\iint_{\Ecal_1}|z|^4/\sigma^2=2dC_1$, $d\iint_{\Ecal_1}\psi_0=C_1/2$, and
\[
M_r[F]=F-\tfrac{C_1}4r^2(\partial_\tau-\Delta)F+o(r^2),\qquad
S_r[h]=\tfrac{C_1}4r^2h+o(r^2).
\]
\end{lemma}

\begin{proof}
Expand $F$ to second order in space and first order in time; odd moments
vanish, isotropy gives $\iint z_iz_j|z|^2/\sigma^2=\delta_{ij}\,2C_1$, and
$\int_0^r\rho\,\dd\rho=r^2/2$. The spatial moment identity follows by applying the homogeneous
Watson formula to the caloric polynomial $|x|^2+2d\tau$.
For the source normalization use $F(x,\tau)=\tau$, which satisfies
$(\partial_\tau-\Delta)F=1$, in the inhomogeneous identity:
$\tau_0=M_r[\tau]+S_r[1]$.  As
$M_r[\tau]=\tau_0-r^2 C_1/4$, we obtain
$d\iint_{\Ecal_1}\psi_0=C_1/2$.
\end{proof}

\begin{theorem}[Residual limit and feasibility]\label{thm:dsm-limit}
(i) For $F\in C^{2,1}$ and continuous $h$,
\[
\lim_{r\to0}\frac{4}{C_1r^2}\mathcal R_r[F;h](x_0,\tau_0)
=\bigl((\partial_\tau-\Delta)F-h\bigr)(x_0,\tau_0),
\]
the pointwise differential residual. Equivalently,
\[
\mathcal R_r[F;h](x_0,\tau_0)=\frac{C_1r^2}{4}
\bigl((\partial_\tau-\Delta)F-h\bigr)(x_0,\tau_0)+o(r^2).
\]
In particular, the $O(r^2)$ statement concerns the \emph{unscaled}
residual, not the error of the scaled residual.
(ii) The true log-density and score satisfy $\mathcal R_r[v;h_v]=0$ and
$\mathcal R_r[s;\nabla h_v]=0$ for every admissible $r$.
\end{theorem}

\begin{proof}
(i) follows from Lemma~\ref{lem:expansion}; (ii) from
Theorem~\ref{thm:hierarchy} and Corollary~\ref{cor:score-mean}.
\end{proof}

%==========================================================
\section{Exact sampling on heat balls}\label{sec:sampling}
%==========================================================

The identities involve the Watson weight $|z|^2/\sigma^2$, the source
weight $\psi_0$, and weights introduced by integration by parts. Their
moment calculations add further powers of $|z|$ and $\sigma$. We show
that these are all members of a single family of measures on the unit
heat ball $\Ecal_1$.

\subsection{A Beta--Gamma family of heat-ball measures}

\begin{theorem}[Exact factorization]\label{thm:family}
Let $a>-d/2$, $b>-1$ and $c<d/2+a+1$, and let $\mu_{a,b,c}$ be the
probability measure on $\Ecal_1$ with density proportional to
$|z|^{2a}\,\psi_0(z,\sigma)^b\,\sigma^{-c}$. Write
$z=\sqrt t\,R(\sigma)\,\hat z$ with $\hat z\in\Sph^{d-1}$, $t\in[0,1]$, and
$Y=-\log(4\pi\sigma)\in(0,\infty)$. Under $\mu_{a,b,c}$ the variables
$\hat z$, $t$, $Y$ are independent and
\begin{equation}\label{eq:family}
  \hat z\sim\mathrm{Unif}(\Sph^{d-1}),\qquad
  t\sim\Beta\bigl(\tfrac d2+a,\ b+1\bigr),\qquad
  Y\sim\GammaDist\bigl(m+1,\ \text{rate }\lambda\bigr),
\end{equation}
where $m:=d/2+a+b$ and $\lambda:=d/2+a-c+1$. The normalizing constant is
\begin{equation}\label{eq:family-mass}
  Z_{a,b,c}:=\iint_{\Ecal_1}|z|^{2a}\psi_0^b\sigma^{-c}\,\dd z\,\dd\sigma
  =\frac{\pi^{d/2}}{\Gamma(d/2)}\,
   \mathrm B\bigl(\tfrac d2+a,b+1\bigr)\,
   \frac{(2d)^{m}}{4^{b}\,(4\pi)^{\lambda}}\,
   \frac{\Gamma(m+1)}{\lambda^{m+1}}.
\end{equation}
\end{theorem}

\begin{proof}
In polar coordinates $\dd z=\rho^{d-1}\dd\rho\,\dd\hat z$ and, with
$\rho=\sqrt t\,R$, $\rho^{d-1}\dd\rho=\tfrac12R^dt^{d/2-1}\dd t$. By
\eqref{eq:psi0}, $\psi_0=R^2(1-t)/(4\sigma)$, so the weight times the volume
element equals
\[
\tfrac12\,4^{-b}\;t^{d/2+a-1}(1-t)^b\;R^{2m}\sigma^{-b-c}\;
\dd t\,\dd\hat z\,\dd\sigma,
\]
a product of a function of $t$, the uniform measure on the sphere and a
function of $\sigma$. The $t$-factor is the $\Beta(d/2+a,b+1)$ kernel. For
the $\sigma$-factor substitute $\sigma=e^{-Y}/(4\pi)$, so that
$R^2=2d\sigma Y$ and $\dd\sigma=\sigma\,\dd Y$ up to orientation:
$R^{2m}\sigma^{-b-c}\dd\sigma=(2d)^mY^m\sigma^{m-b-c+1}\dd Y
=(2d)^m(4\pi)^{-\lambda}Y^me^{-\lambda Y}\dd Y$, since
$m-b-c+1=\lambda$. This is the $\GammaDist(m+1,\text{rate }\lambda)$
kernel, finite because $\lambda>0$. Multiplying the three normalizing
constants, with $|\Sph^{d-1}|=2\pi^{d/2}/\Gamma(d/2)$, gives
\eqref{eq:family-mass}.
\end{proof}

\begin{remark}[Why the common measure family is useful]\label{rem:family-structure}
The same factorization also simplifies changes of weight. Multiplying
an integrand by
$|z|^{2k}\psi_0^\ell\sigma^{-j}$ shifts the parameters
$(a,b,c)$ to $(a+k,b+\ell,c+j)$ whenever both masses are finite:
\begin{equation}\label{eq:parameter-shift}
 \E_{\mu_{a,b,c}}\!\left[|z|^{2k}\psi_0^\ell\sigma^{-j}G\right]
 =\frac{Z_{a+k,b+\ell,c+j}}{Z_{a,b,c}}
   \E_{\mu_{a+k,b+\ell,c+j}}[G].
\end{equation}
Watson averaging, source integration, integration by parts, and moment
calculations can therefore use the same angular, Beta, and Gamma samplers
with different parameters. In particular, changing the source weight does
not require designing a new proposal distribution.
\end{remark}

\begin{corollary}[The three measures of the hierarchy]\label{cor:three}
\begin{enumerate}
\item The Watson measure $\mu_\kappa:=\mu_{1,0,2}$, with density
$\kappa=\frac14|z|^2\sigma^{-2}$ ($Z_{1,0,2}=4$):
$t\sim\Beta(\frac d2+1,1)$, $Y\sim\GammaDist(\frac d2+2,\text{rate }\frac d2)$.
\item The source measure $\mu_\psi:=\mu_{0,1,0}$
($Z_{0,1,0}=C_1/(2d)$): $t\sim\Beta(\frac d2,2)$,
$Y\sim\GammaDist(\frac d2+2,\text{rate }\frac d2+1)$.
\item The uniform measure $\mu_{\mathrm U}:=\mu_{0,0,0}$
($Z_{0,0,0}=|\Ecal_1|$): $t\sim\Beta(\frac d2,1)$,
$Y\sim\GammaDist(\frac d2+1,\text{rate }\frac d2+1)$.
\end{enumerate}
Consequently, with $U\sim\mathrm{Unif}(0,1)$ independent,
\begin{align}
  M_r[F](x_0,\tau_0)&=\E_{\mu_\kappa}\bigl[F(x_0+rz,\tau_0-r^2\sigma)\bigr],
  \label{eq:Mr-exact}\\
  S_r[h](x_0,\tau_0)&=\frac{C_1r^2}4\,\E_{\mu_\psi}\bigl[h(x_0+r\sqrt U z,\
  \tau_0-r^2U\sigma)\bigr].\label{eq:Sr-exact}
\end{align}
\end{corollary}

\begin{proof}
Read off~\eqref{eq:family}--\eqref{eq:family-mass}. For~\eqref{eq:Sr-exact}
substitute $\rho=r\sqrt U$, so that $\rho\,\dd\rho=\frac{r^2}2\dd U$, and use
$dZ_{0,1,0}=C_1/2$ (Lemma~\ref{lem:expansion}).
\end{proof}

\begin{algorithm}[ht]
\caption{Exact sampler for $\mu_{a,b,c}$ on the heat ball of radius $r$}
\label{alg:family}
\textbf{Input:} $d$, exponents $(a,b,c)$, radius $r$, anchor $(x_0,\tau_0)$, sample size $N$.\\
\textbf{Output:} i.i.d.\ points $(x_i,\tau_i)$ whose rescaled preimages $(z_i,\sigma_i)$ follow $\mu_{a,b,c}$.\\[2pt]
\For{$i=1,\dots,N$}{
  Draw $Y_i\sim\GammaDist(d/2+a+b+1,\ \text{rate }d/2+a-c+1)$; set $\sigma_i=e^{-Y_i}/(4\pi)$\;
  Draw $t_i\sim\Beta(d/2+a,\ b+1)$; set $R_i=\sqrt{2d\,\sigma_iY_i}$\;
  Draw $g_i\sim\mathcal N(0,I_d)$; set $z_i=\sqrt{t_i}\,R_i\,g_i/|g_i|$\;
  Return $(x_i,\tau_i)=(x_0+rz_i,\ \tau_0-r^2\sigma_i)$\;
}
\end{algorithm}

For the Watson measure, $t\sim\Beta(\frac d2+1,1)$ can be drawn as
$V^{2/(d+2)}$ with $V$ uniform. The cost per sample is one Gamma draw, one
Beta draw and one $d$-dimensional Gaussian, with no rejection and unit
weight.

\subsection{Closed-form constants and concentration}

\begin{corollary}[Closed-form constants]\label{cor:constants}
With $\sigma^\star:=1/(4\pi e)$,
\begin{gather*}
|\Ecal_1|=\frac{(d/2)^{d/2}}{4\pi(d/2+1)^{d/2+1}},\qquad
C_1=4\E_{\mu_\kappa}[\sigma]=\frac1\pi\Bigl(\frac d{d+2}\Bigr)^{d/2+2}
\ \nearrow\ \frac1{\pi e},\\
\E_{\mu_\kappa}|z|^2=\frac{dC_1}2,\qquad
\E_{\mu_\kappa}[\sigma^2]=\frac1{16\pi^2}\Bigl(\frac d{d+4}\Bigr)^{d/2+2},\qquad
\E_{\mu_\kappa}\Bigl[\frac{|z|^2}{2\sigma}\Bigr]=d+2,\\
R_{\max}:=\max_\sigma R(\sigma)=R(\sigma^\star)=\sqrt{d/(2\pi e)}.
\end{gather*}
\end{corollary}

\begin{proof}
$|\Ecal_1|=Z_{0,0,0}$ by~\eqref{eq:family-mass}. The Gamma Laplace transform
$\E[e^{-kY}]=(\lambda/(\lambda+k))^{m+1}$ with $k=1,2$ gives the moments of
$\sigma$. Further, $|z|^2=tR^2=2d\,t\,\sigma Y$ with $t$, $Y$ independent,
and $\E_{\mu_\kappa}|z|^2=\frac14\iint|z|^4/\sigma^2=dC_1/2$ by
Lemma~\ref{lem:expansion}; $|z|^2/(2\sigma)=d\,tY$ has mean
$d\cdot\frac{d+2}{d+4}\cdot\frac{d+4}d$. Finally
$-\sigma\log(4\pi\sigma)$ is maximized at $\sigma^\star$, with value
$\sigma^\star$.
\end{proof}

\begin{corollary}[Concentration of the Watson measure]\label{cor:concentration}
Under $\mu_\kappa$, as $d\to\infty$: $d(1-t)\Rightarrow\chi^2_2$, so
$|z|=R(\sigma)(1-O_P(d^{-1}))$; and
$\sqrt d\,(\sigma-\sigma^\star)\Rightarrow\mathcal N(0,2(\sigma^\star)^2)$.
The measure thus concentrates on a thin shell of radius
$R_{\max}=\sqrt{d/(2\pi e)}$ at backward time $\sigma^\star$.
\end{corollary}

\begin{proof}
$t=V^{2/(d+2)}=e^{-2E/(d+2)}$ with $E\sim\mathrm{Exp}(1)$, so
$d(1-t)\to2E\sim\chi^2_2$. $Y\sim\GammaDist(\frac d2+2,\text{rate }\frac d2)$
has mean $1+4/d$ and variance $2(d+4)/d^2$, so $\sqrt d(Y-1)\Rightarrow\mathcal
N(0,2)$; apply the delta method to $\sigma=\sigma^\star e^{-(Y-1)}$.
\end{proof}

\subsection{Importance sampling from a bounding cylinder}

\begin{proposition}[Variance gap]\label{prop:variance-gap}
Let $S=\{(z,\sigma):0<\sigma<1/(4\pi),\ |z|\le R_{\max}\}$ be the bounding
cylinder of $\Ecal_1$, $q$ the uniform density on $S$ and $w=\kappa/q$. For
$\vartheta\in L^2(\mu_\kappa)$ the direct estimator (from $\mu_\kappa$) and
the importance-sampling estimator (from $q$) are unbiased. When the
importance-sampling second moment is finite, their variances satisfy
\[
N\bigl(\mathrm{Var}[\widehat\vartheta^{\,q}_N]-\mathrm{Var}[\widehat\vartheta^{\,\kappa}_N]\bigr)
=\E_{\mu_\kappa}[(w-1)\vartheta^2],
\]
which depends on $\vartheta$ and can have either sign. When the
importance-sampling second moment is infinite, the comparison is
understood in the extended sense rather than as a finite variance bound.
For a nonzero constant $\vartheta$,
\[
\chi^2(\mu_\kappa\Vert q)+1\ \ge\ \frac{|S|}{|\Ecal_1|}
=\frac{(d/2+1)^{d/2+1}e^{-d/2}}{\Gamma(d/2+1)}\ \sim\ \frac{e}{2\sqrt\pi}\sqrt d.
\]
\end{proposition}

\begin{proof}
Unbiasedness follows from $wq=\kappa$ and
$\vartheta\in L^1(\mu_\kappa)$; the variance identity, when finite,
follows from $w^2q=w\kappa$. For the bound,
Cauchy--Schwarz on $\Ecal_1$ gives $\E_q[w^2]\ge1/q(\Ecal_1)=|S|/|\Ecal_1|$;
use $|S|=\omega_dR_{\max}^d/(4\pi)$, where $\omega_d=\pi^{d/2}/\Gamma(d/2+1)$
is the volume of the unit ball, Corollary~\ref{cor:constants} and Stirling's
formula.
\end{proof}

A proposal uniform on $\Ecal_1$ itself has $\chi^2=O(1)$, so the gap is
specific to over-dispersed bounding proposals.

%==========================================================
\section{Estimators and error analysis}\label{sec:error}
%==========================================================

The factorization gives unit-weight Monte Carlo estimators for each
identity. We establish their unbiasedness, variance, and cost bounds,
including the dependence on dimension and heat-ball radius. We then
compare them with bounding-cylinder importance sampling and Gaussian
semigroup estimators. Section~\ref{sec:experiments} tests these results.

\subsection{Unbiased estimators}
Let $F$ solve $\partial_\tau F-\Delta F=h$ near the closed heat ball of
radius $r$ at $(x_0,\tau_0)$. Draw
$(z_i,\sigma_i)\sim\mu_\kappa$, $i\le N$, and independently
$(z'_j,\sigma'_j)\sim\mu_\psi$, $U_j\sim\mathrm{Unif}(0,1)$, $j\le N'$, and set
\begin{equation}\label{eq:estimator}
  \widehat F_{N,N'}:=\frac1N\sum_{i=1}^NF(x_0+rz_i,\tau_0-r^2\sigma_i)
  +\frac{C_1r^2}{4}\,\frac1{N'}\sum_{j=1}^{N'}
   h\bigl(x_0+r\sqrt{U_j}z'_j,\ \tau_0-r^2U_j\sigma'_j\bigr).
\end{equation}

\begin{proposition}[Unbiasedness and variance]\label{prop:variance}
$\E\widehat F_{N,N'}=F(x_0,\tau_0)$ and
$\mathrm{Var}\,\widehat F_{N,N'}=V_M/N+(C_1r^2/4)^2V_S/N'$, where $V_M$ and
$V_S$ are the variances of a single term of each sum. If
$|F(x,\tau)-F(x_0,\tau_0)|\le L|x-x_0|+L_\tau|\tau-\tau_0|$ on the heat ball
and $|h|\le H$ there, then
\begin{equation}\label{eq:var-bound}
  \E\bigl(\widehat F_{N,N'}-F(x_0,\tau_0)\bigr)^2
  \le\frac1N\Bigl(L^2r^2\,dC_1+\frac{L_\tau^2r^4}{8\pi^2}\Bigr)
  +\frac{C_1^2r^4H^2}{16\,N'}
  \le\frac1N\Bigl(\frac{L^2r^2d}{\pi e}+\frac{L_\tau^2r^4}{8\pi^2}\Bigr)
  +\frac{r^4H^2}{16\pi^2e^2N'}.
\end{equation}
\end{proposition}

\begin{proof}
Unbiasedness is Corollary~\ref{cor:three} with Lemma~\ref{lem:monotone}.
For the first sum,
$V_M\le\E_{\mu_\kappa}(Lr|z|+L_\tau r^2\sigma)^2
\le2L^2r^2\E|z|^2+2L_\tau^2r^4\E\sigma^2$; insert
Corollary~\ref{cor:constants} and $\E\sigma^2\le1/(16\pi^2)$. For the second,
$V_S\le H^2$. Finally $C_1<1/(\pi e)$.
\end{proof}

\begin{proposition}[Cost-optimal allocation]\label{prop:allocation}
Let $\alpha=C_1r^2/4$, and let positive constants $c_M,c_S$ be the costs
of one Watson and one source evaluation, respectively. For $V_M,V_S>0$,
under the continuous budget constraint $c_MN+c_SN'=B$, the minimum of
$V_M/N+\alpha^2V_S/N'$ is attained at
\begin{equation}\label{eq:optimal-allocation}
 \frac{N}{N'}=\sqrt{\frac{V_Mc_S}{\alpha^2V_Sc_M}},\qquad
 \min\mathrm{Var}\,\widehat F
 =\frac{\bigl(\sqrt{c_MV_M}+\alpha\sqrt{c_SV_S}\bigr)^2}{B}.
\end{equation}
Integer allocations are obtained by rounding while respecting the budget.
If either variance vanishes, no samples of that random term are required
when its expectation is known exactly; otherwise retain at least one
sample per nonconstant term.
\end{proposition}
\begin{proof}
Set $A=V_M$, $D=\alpha^2V_S$. By Cauchy--Schwarz,
$(A/N+D/N')(c_MN+c_SN')\ge
(\sqrt{Ac_M}+\sqrt{Dc_S})^2$; equality gives the stated ratio.
\end{proof}
\begin{remark}[Dimension and relative accuracy]\label{rem:effective-complexity}
The complexity statement in Corollary~\ref{cor:complexity} counts
function evaluations with $L$ held fixed. In applications, $L$, $H$,
the cost of one evaluation, and the target magnitude may all depend on
$d$. An absolute error bound may also say little about relative accuracy
when $u$ is very small. We therefore report density, log-density, and
score errors separately in E4. The allocation formula can be used with
pilot estimates of $V_M$, $V_S$, $c_M$, and $c_S$.
\end{remark}

\begin{corollary}[Complexity]\label{cor:complexity}
A root-mean-square error $\varepsilon$ is reached with
$N=\bigl\lceil2\varepsilon^{-2}(L^2r^2d/(\pi e)+L_\tau^2r^4/(8\pi^2))\bigr\rceil$
Watson samples and
$N'=\bigl\lceil r^4H^2/(8\pi^2e^2\varepsilon^2)\bigr\rceil$ source samples,
i.e.\ $O(dr^2L^2/\varepsilon^2)$ evaluations of $F$, linear in the dimension
for fixed spatial Lipschitz constant.
\end{corollary}

\begin{remark}[Choice of radius in high dimension]\label{rem:radius}
For $F=u=e^v$ the Watson integrand is
$u(x_0+rz,\cdot)/u(x_0,\tau_0)\approx\exp(r\,z\cdot\nabla v)$, and since
$\E_{\mu_\kappa}(z\cdot g)^2=|g|^2\E|z|^2/d=|g|^2C_1/2$, the exponent has
variance about $r^2|\nabla v|^2C_1/2$. For a density in $\R^d$ the score
typically has $|\nabla v|^2$ of order $d$ (e.g.\ $d/\sigma^2$ for a Gaussian of
variance $\sigma^2$), so at fixed $r$ the integrand becomes log-normally
heavy-tailed as $d$ grows, and the bound~\eqref{eq:var-bound}, though valid,
is not informative. Two consequences, confirmed in E1 and E4: in high
dimension one should evaluate $v=\log u$ and $s=\nabla v$, whose integrands
vary only polynomially on the heat ball, rather than $u$; and when $u$ itself
is needed the radius should be scaled as $r\propto d^{-1/2}$.
\end{remark}

\begin{remark}[Comparison with the Gaussian-semigroup representation]
\label{rem:semigroup}
On $\R^d$ the Duhamel formula
$F(x_0,\tau_0)=\E[F(x_0+\sqrt{2\ell}\,G,\tau_0-\ell)]
+\int_0^\ell\E[h(x_0+\sqrt{2\ell'}G,\tau_0-\ell')]\,\dd\ell'$, $G\sim\mathcal N(0,I)$,
also gives unbiased, unit-weight estimators. At matched mean time lag
$\ell=r^2\E_{\mu_\kappa}\sigma$, its first term has the bound
$2L^2\ell d=L^2r^2dC_1/2$ in place of $L^2r^2\E|z|^2=L^2r^2dC_1/2$: the
leading-order variance bounds coincide, and by
Corollary~\ref{cor:concentration} both sampling laws concentrate on the same
shell in high dimension. Table~\ref{tab:semigroup} confirms numerically that
the variances agree to within $13\%$ up to $d=256$, and E2 and E6 extend this
to $d=1024$ and to problems with a source term. The heat ball is
therefore \emph{not} a variance-reduction device relative to the semigroup.
Its advantages are structural: all samples lie in a compact set of spatial
radius $rR_{\max}$ and time lag at most $r^2/(4\pi)$, so only local values
of $F$ and $h$ are needed; the identity holds verbatim inside any domain
that contains the heat ball, with no boundary treatment; and the source
integral is sampled exactly on the same compact set.
\end{remark}

\begin{table}[ht]
\centering
\caption{Per-sample variance of the heat-ball estimator ($r=1$) and of the
Gaussian-semigroup estimator at matched mean time lag, for the caloric
function $u(x,\tau)\propto\exp(-|x|^2/(2(1+2\tau)))$ normalized so that
$u(x_0,\tau_0)=1$, with $|x_0|=1$, $\tau_0=0.5$, $4\times10^5$ samples.
Both estimators are unbiased (sample means within $3\times10^{-4}$ of $1$).
Computed with the accompanying script \texttt{variance\_check.py}.}
\label{tab:semigroup}
\begin{tabular}{rccc}
\toprule
$d$ & heat ball & Gaussian semigroup & ratio \\
\midrule
2   & $5.17\times10^{-3}$ & $5.05\times10^{-3}$ & 1.02 \\
16  & $1.75\times10^{-2}$ & $1.68\times10^{-2}$ & 1.04 \\
64  & $4.25\times10^{-2}$ & $3.95\times10^{-2}$ & 1.08 \\
256 & $1.45\times10^{-1}$ & $1.29\times10^{-1}$ & 1.13 \\
\bottomrule
\end{tabular}
\end{table}

\subsection{Squared residuals without variance bias}\label{sec:ustat}
Residual-based losses square a Monte Carlo estimate. If $\widehat{\mathcal R}$
is unbiased for $\mathcal R$, then
$\E\widehat{\mathcal R}^2=\mathcal R^2+\mathrm{Var}\,\widehat{\mathcal R}$,
a bias that favours over-smooth candidates; the same effect is analysed for
derivative-free losses in~\cite{HanNicaStinchcombe2020}. With two independent
copies $\widehat{\mathcal R}^{(1)},\widehat{\mathcal R}^{(2)}$ built from
independent samples, $\E[\widehat{\mathcal R}^{(1)}\widehat{\mathcal R}^{(2)}]
=\mathcal R^2$ exactly.

%==========================================================
\section{Heat-ball residuals for learned score models}\label{sec:score-models}
%==========================================================

A candidate log-density or score field need not satisfy the identities
exactly. Its defect gives a local residual, whose small-radius limit is
the corresponding differential residual (Theorem~\ref{thm:dsm-limit}).
Unlike the pointwise log-density residual, the heat-ball form does not
require a network Laplacian. Integration by parts removes some further
network derivatives exactly; finite-difference replacements introduce a
truncation error. These alternatives can increase Monte Carlo variance,
and were not uniformly faster than automatic differentiation in E7--E8.
They remain useful when higher-order network derivatives are difficult
to obtain, and for accurate evaluation with product quadrature.

\subsection{Setting}
Let the forward process~\eqref{eq:gen_sde} have an arbitrary drift---linear
as in DDPM and SMLD, or nonlinear and learned as in diffusion
Schr\"odinger bridges~\cite{DeBortoli2021} and neural flow diffusion
models~\cite{Bartosh2024}---with $\ft$ and $\nabla\cdot\ft$ computable (for
example in closed form, as for a drift built from a potential and a
divergence-free part). A score model is either a scalar log-density network
$v_\theta(x,\tau)$ with $s_\theta=\nabla_xv_\theta$, or a vector network
$s_\theta(x,\tau)$. For nonlinear drifts the transition kernel is not
Gaussian, so denoising score-matching targets are not available in closed
form; the Fokker--Planck equation still characterizes the true score.

\subsection{Heat-ball residuals}
For a scalar model, the candidate source is
$h_v[\theta]=|s_\theta|^2-\ft\cdot s_\theta-\nabla\cdot\ft$, and the residual is
\begin{equation}\label{eq:res-scalar}
  \mathcal R_r[v_\theta](x_0,\tau_0)=v_\theta(x_0,\tau_0)-M_r[v_\theta]-S_r\bigl[h_v[\theta]\bigr].
\end{equation}
The transport part of $S_r$ can also be written without $s_\theta$: since
$\psi_0$ vanishes on $\partial\Ecal_1$ and $\nabla_z\psi_0=-z/(2\sigma)$,
integration by parts in $z$ gives
\begin{equation}\label{eq:transport-ibp}
  -S_r[\ft\cdot s_\theta]
  =d\iint_{\Ecal_1}\!\int_0^r\Bigl(\rho\,v_\theta\,(\nabla\cdot\ft)\,\psi_0
  -v_\theta\,\ft\cdot\frac z{2\sigma}\Bigr)\dd\rho\,\dd z\,\dd\sigma,
\end{equation}
evaluated at $(x_0+\rho z,\tau_0-\rho^2\sigma)$; the first term is sampled
from $\mu_\psi$ and the second, best, from $\mu_{1/2,0,1}$ with antithetic
pairing (Proposition~\ref{prop:ibp}).
For a vector model, Corollary~\ref{cor:score-mean} gives
\begin{equation}\label{eq:res-vector}
  \mathcal R_r[s_\theta]=s_\theta-M_r[s_\theta]
  -S_r\bigl[(\nabla s_\theta)^{\!\top}(2s_\theta-\ft)-(\nabla\ft)^{\!\top}s_\theta
  -\nabla(\nabla\cdot\ft)\bigr].
\end{equation}

\begin{proposition}[Properties and cost]\label{prop:residual-cost}
\begin{enumerate}
\item \emph{Exactness.} The true log-density and score give zero residual
at every admissible radius (Theorem~\ref{thm:dsm-limit}(ii)).
\item \emph{Consistency.} As $r\to0$, $4\mathcal R_r/(C_1r^2)$ converges to
the pointwise Fokker--Planck residual $(\partial_\tau-\Delta)v_\theta-h_v[\theta]$
(scalar case) or $(\partial_\tau-\Delta)s_\theta-\nabla h_v[\theta]$ (vector
case), which for gradient fields is the score Fokker--Planck residual
enforced by FP-Diffusion~\cite{Lai2023} (Theorem~\ref{thm:dsm-limit}(i)).
\item \emph{Unbiased evaluation.} With the samplers of
Corollary~\ref{cor:three}, each residual is estimated without bias, and its
square without bias by the two-sample product of Section~\ref{sec:ustat}.
\item \emph{No second derivatives.} The scalar residual needs values of
$v_\theta$ at $N$ Watson points and of $v_\theta$ and $s_\theta=\nabla v_\theta$
at $N'$ source points: first derivatives only. The vector residual needs
$s_\theta$ and one vector--Jacobian product
$(\nabla s_\theta)^{\!\top}(2s_\theta-\ft)$ per source point (none in the form of
Proposition~\ref{prop:score-ibp}, Section~\ref{sec:derivfree}). Neither
requires a divergence, trace or Laplacian of the network. The pointwise
residual, by contrast, needs $\Delta v_\theta$ or $\nabla(\nabla\cdot s_\theta)$,
whose exact evaluation costs $O(d)$ backward passes and whose Hutchinson
estimate~\cite{Hutchinson1989} adds variance.
\end{enumerate}
\end{proposition}

\begin{proof}
Items 1--3 restate Theorem~\ref{thm:dsm-limit}, Corollary~\ref{cor:three}
and Section~\ref{sec:ustat}. Item~4 follows by inspecting
\eqref{eq:res-scalar}--\eqref{eq:res-vector}: $\nabla(\nabla\cdot\ft)$ and
$\nabla\ft$ concern the known drift, and
$(\nabla s_\theta)^{\!\top}\xi$ is one reverse-mode pass for a given vector
$\xi$.
\end{proof}

\begin{remark}[Scope]\label{rem:scope}
The residual is a necessary condition: it vanishes for the true score, but
we do not claim that minimizing it alone identifies the score, nor that it
improves sample quality when added to a training objective; constant
functions, for instance, have residual $-S_r[-\nabla\cdot\ft]$, which
vanishes only for divergence-free drifts. Its intended uses are (a) an
a~posteriori, local consistency diagnostic for a trained model under any
forward drift, and (b) a regularizer when denoising targets are
unavailable. Experiment~E9 tests (a); (b) is studied in~\cite{CompanionML}.
\end{remark}

\begin{remark}[Normalization]\label{rem:normalization}
Both residuals are invariant under $v_\theta\mapsto v_\theta+c$ for a
constant $c$ (equivalently $u_\theta\mapsto e^cu_\theta$), because $M_r$
reproduces constants and $h_v$ depends on $v$ only through $\nabla v$. They
are not invariant under a time-dependent offset $c(\tau)$, which adds
$c(\tau_0)-\E_{\mu_\kappa}[c(\tau_0-r^2\sigma)]$ to the residual. An
unnormalized log-density model must therefore either be normalized or
carry its own offset $c_\theta(\tau)$.
\end{remark}

\subsection{Replacing network derivatives by values}\label{sec:derivfree}

The residuals above retain first network derivatives. Differentiating a
loss through these terms can require second-order automatic
differentiation. We give value-based alternatives and examine their cost
in E7. For a function $F$ and direction $\xi$, define
\begin{equation}\label{eq:central}
  \delta_\varepsilon F(y;\xi):=\frac{F(y+\varepsilon\xi)-F(y-\varepsilon\xi)}{2\varepsilon}
  =\nabla F(y)\cdot\xi+\frac{\varepsilon^2}{6}D^3F(y)[\xi,\xi,\xi]+o(\varepsilon^2).
\end{equation}

\paragraph{Exact identities versus approximate evaluations.}
The integration-by-parts identities below are exact at the continuum level
and admit unbiased Monte Carlo evaluation when their integrands are computed
exactly. Central differences with a fixed step $\varepsilon>0$, however,
introduce an $O(\varepsilon^2)$ deterministic bias under adequate
smoothness. Independent sample products remove the variance bias of
squaring a Monte Carlo residual only for the expectation actually being
sampled; they do not cancel finite-difference truncation error.

\paragraph{Squared gradient norm.}
The source $h_v$ contains $|\nabla v|^2$. Estimating $\nabla v$ by a
random-direction gradient estimator and squaring it is a poor choice: a
gradient estimate from $h$ directions has relative error of order
$\sqrt{d/h}$, and the product of two independent ones (needed for
unbiasedness) has relative variance growing with $d$. Squaring a single
directional difference is much better.

\begin{proposition}[Squared-difference estimator]\label{prop:sqnorm}
Let $d\ge2$, $\xi\sim\mathrm{Unif}(\Sph^{d-1})$ and $v\in C^3$ near $y$. Then
\[
\E_\xi\bigl[d\,\delta_\varepsilon v(y;\xi)^2\bigr]=|\nabla v(y)|^2+O(\varepsilon^2),
\qquad
\lim_{\varepsilon\to0}\frac{\mathrm{Var}_\xi\bigl[d\,\delta_\varepsilon v(y;\xi)^2\bigr]}{|\nabla v(y)|^4}
=\frac{2(d-1)}{d+2}<2 .
\]
The variance ratio is asserted only when $|\nabla v(y)|>0$.
With $k$ independent directions the variance is divided by $k$; its
relative version has the stated interpretation for a nonzero gradient.
The estimator costs two evaluations of $v$ per direction.
\end{proposition}

\begin{proof}
By~\eqref{eq:central}, $d\,\delta_\varepsilon v^2=d(\xi\cdot g)^2+O(\varepsilon^2)$
with $g=\nabla v(y)$. For $\xi$ uniform on the sphere,
$X=(\xi\cdot g)^2/|g|^2\sim\Beta(\frac12,\frac{d-1}2)$, so $\E X=1/d$ and
$\E X^2=3/(d(d+2))$; hence $\mathrm{Var}(dX)=2(d-1)/(d+2)$.
\end{proof}

\paragraph{Transport terms and Jacobian products.}
The term $\ft\cdot\nabla v$ is a directional derivative and is replaced
deterministically by $|\ft|\,\delta_\varepsilon v(y;\ft/|\ft|)$, with error
$O(\varepsilon^2)$ and no added variance. In the vector
residual~\eqref{eq:res-vector}, $(\nabla s)^{\!\top}\xi$ with
$\xi=2s-\ft$ is replaced by $|\xi|\,\delta_\varepsilon s(y;\xi/|\xi|)=(\nabla s)\xi+O(\varepsilon^2)$.
This is a Jacobian--vector product rather than a vector--Jacobian product; the
two coincide when $s$ is a gradient field, so the corresponding \emph{continuum} residual vanishes for the true score;
the finite-$\varepsilon$ replacement in general does not vanish exactly.
Theorem~\ref{thm:dsm-limit} applies to the continuum form with
the source $(\nabla s)(2s-\ft)-(\nabla\ft)^{\!\top}s-\nabla(\nabla\cdot\ft)$.

\paragraph{Divergences.}
Divergence terms are removed by integration by parts against a member of
the Beta--Gamma family, and an antithetic pairing restores the right
variance.

\begin{proposition}[Antithetic integration by parts]\label{prop:ibp}
Let $G$ be a $C^1$ vector field on a neighbourhood of the closed heat ball of
radius $r$ at $(x_0,\tau_0)$. Let $(z,\sigma)\sim\mu_{1/2,0,1}$,
$\hat z=z/|z|$, $U\sim\mathrm{Unif}(0,1)$ independent, $\rho=rU$ and
$\tau=\tau_0-\rho^2\sigma$. Then
\begin{equation}\label{eq:ibp-antithetic}
  S_r[-\nabla\cdot G](x_0,\tau_0)
  =-\frac{d\,r\,Z_{1/2,0,1}}{2}\;
  \E\Bigl[\tfrac12\bigl(G(x_0+\rho z,\tau)-G(x_0-\rho z,\tau)\bigr)\cdot\hat z\Bigr],
\end{equation}
with $Z_{1/2,0,1}$ from~\eqref{eq:family-mass}. The integrand is $O(r)$
uniformly, so the estimator has standard deviation $O(r^2)$, the same order
as its mean; without the antithetic pairing the standard deviation is
$O(r)$.
\end{proposition}

\begin{proof}
Since $\rho\,(\nabla_x\cdot G)(x_0+\rho z)=\nabla_z\cdot[G(x_0+\rho z)]$,
$\psi_0=0$ on $\partial\Ecal_1$ and $\nabla_z\psi_0=-z/(2\sigma)$,
\[
S_r[-\nabla\cdot G]=-d\int_0^r\!\iint_{\Ecal_1}\psi_0\,\nabla_z\cdot G\,\dd z\,\dd\sigma\,\dd\rho
=-d\int_0^r\!\iint_{\Ecal_1}G\cdot\frac{z}{2\sigma}\,\dd z\,\dd\sigma\,\dd\rho .
\]
Writing $z/(2\sigma)=\frac12|z|\sigma^{-1}\hat z$ identifies the weight of
$\mu_{1/2,0,1}$, and $\rho=rU$ gives~\eqref{eq:ibp-antithetic} without the
pairing. The law $\mu_{1/2,0,1}$ is invariant under $z\mapsto-z$, which
flips $\hat z$, so averaging the two reflected samples leaves the mean
unchanged. Finally $\frac12(G(x_0+\rho z)-G(x_0-\rho z))=\rho\,(\nabla G)z+O(\rho^3)$.
\end{proof}

\begin{corollary}[A completely derivative-free density residual]\label{cor:density-free}
With $G=\ft u_\theta$, the density residual
$\mathcal R_r[u_\theta]=u_\theta(x_0,\tau_0)-M_r[u_\theta]-S_r[-\nabla\cdot(\ft u_\theta)]$
is estimated without bias from values of $u_\theta$ and $\ft$ only,
by~\eqref{eq:Mr-exact} and~\eqref{eq:ibp-antithetic}. The same applies to
$S_r[-\nabla\cdot\ft]$ in $h_v$ and to $-\nabla\cdot(\ft w)$ in $h_w$, which
removes all derivatives of a learned drift.
\end{corollary}

The source of the score identity can be integrated by parts as a whole, which
removes every derivative of the score from the vector residual.

\begin{proposition}[Score identity without derivatives of the score]\label{prop:score-ibp}
Let $h$ be $C^1$ on a neighbourhood of the closed heat ball of radius $r$ at
$(x_0,\tau_0)$. Then, componentwise,
\begin{equation}\label{eq:score-ibp}
\begin{aligned}
 S_r[\nabla h](x_0,\tau_0)
 &=d\int_0^r\!\!\iint_{\Ecal_1}\frac{z}{2\sigma}\,
 h(x_0+\rho z,\tau_0-\rho^2\sigma)\,\dd z\,\dd\sigma\,\dd\rho\\
 &=\frac{d\,r\,Z_{1/2,0,1}}{2}\;\E\Bigl[\tfrac12
 \bigl(h(x_0+\rho z,\tau)-h(x_0-\rho z,\tau)\bigr)\,\hat z\Bigr].
\end{aligned}
\end{equation}
with $(z,\sigma)\sim\mu_{1/2,0,1}$, $\hat z=z/|z|$, $\rho=rU$, $U\sim\mathrm{Unif}(0,1)$ and
$\tau=\tau_0-\rho^2\sigma$. In particular the score satisfies
\begin{equation}\label{eq:score-mean-ibp}
  s(x_0,\tau_0)=M_r[s](x_0,\tau_0)+\frac{d\,r\,Z_{1/2,0,1}}{2}\;
  \E\Bigl[\tfrac12\bigl(h_v(x_0+\rho z,\tau)-h_v(x_0-\rho z,\tau)\bigr)\,\hat z\Bigr],
  \qquad h_v=|s|^2-\ft\cdot s-\nabla\cdot\ft,
\end{equation}
and for every $C^1$ vector field $s_\theta$---not only the true score---the vector
residual~\eqref{eq:res-vector} equals
\begin{equation}\label{eq:res-vector-ibp}
  \mathcal R_r[s_\theta]=s_\theta(x_0,\tau_0)-M_r[s_\theta]
  -\frac{d\,r\,Z_{1/2,0,1}}{2}\;\E\Bigl[\tfrac12\bigl(h_v[\theta](x_0+\rho z,\tau)-h_v[\theta](x_0-\rho z,\tau)\bigr)\,\hat z\Bigr],
\end{equation}
$h_v[\theta]=|s_\theta|^2-\ft\cdot s_\theta-\nabla\cdot\ft$. This form needs values
of $s_\theta$ at $N$ Watson points and $2N'$ source points and values of
$\ft$ and $\nabla\cdot\ft$; it needs no derivative of $s_\theta$, of $\ft$
or of $\nabla\cdot\ft$. The integrand is $O(r)$ uniformly, so, as
in Proposition~\ref{prop:ibp}, the estimator has standard deviation
$O(r^2)$, the order of its mean.
\end{proposition}

\begin{proof}
Since $\rho\,(\partial_{y_i}h)(x_0+\rho z,\tau_0-\rho^2\sigma)=\partial_{z_i}[h(x_0+\rho z,\tau_0-\rho^2\sigma)]$
(the time argument does not depend on $z$), $\psi_0=0$ on $\partial\Ecal_1$ and
$\nabla_z\psi_0=-z/(2\sigma)$,
\[
  S_r[\partial_ih]=d\int_0^r\!\!\iint_{\Ecal_1}\psi_0\,\partial_{z_i}h\,\dd z\,\dd\sigma\,\dd\rho
  =d\int_0^r\!\!\iint_{\Ecal_1}\frac{z_i}{2\sigma}\,h\,\dd z\,\dd\sigma\,\dd\rho ,
\]
the boundary term vanishing with $\psi_0$. This is Proposition~\ref{prop:ibp}
with $G=-h\,e_i$, which gives the second form of~\eqref{eq:score-ibp}.
Corollary~\ref{cor:score-mean} with $h=h_v$ gives~\eqref{eq:score-mean-ibp}.
The integration by parts holds for every $C^1$ function $h$, in particular for
$h_v[\theta]$, so~\eqref{eq:res-vector} and~\eqref{eq:res-vector-ibp} are the
same quantity; only the estimator differs.
\end{proof}

For a vector score network, \eqref{eq:res-vector-ibp} replaces the
vector--Jacobian product of~\eqref{eq:res-vector} by one extra network
evaluation; for a scalar network used through $s_\theta=\nabla v_\theta$ it
needs first derivatives of $v_\theta$ only, instead of a Hessian--vector
product, which lowers the order of differentiation in training by one.
Together with Proposition~\ref{prop:ibp} for the density and the
transport form~\eqref{eq:transport-ibp} for the log-density, every identity of
the hierarchy can thus be evaluated with values of the unknown alone, with one
exception: the term $|\nabla v|^2$ of the log-density identity. It is the
price of the logarithm ($v=\log u$ turns the linear density equation into a
quadratic one) and cannot be moved by integration by parts, which would
produce $\Delta v$; it is replaced by the squared difference of
Proposition~\ref{prop:sqnorm}, or avoided by working with $u$ or with the score.
For a learned drift the score form still needs values of $\nabla\cdot\ft$ (inside $h_v$
it is not in divergence form, and a second integration by parts against the weight
$z/(2\sigma)$, which does not vanish on $\partial\Ecal_1$, would produce a boundary term); this
is a property of the frozen forward process (for the Helmholtz parameterization,
$\nabla\cdot\ft=\Delta\phi$), not of the score model, and in the density and log-density
forms it is removed by Corollary~\ref{cor:density-free}.

\begin{remark}[Deterministic quadrature]\label{rem:quadrature}
By Theorem~\ref{thm:family}, an expectation under $\mu_{a,b,c}$ is a product
integral: $z=\sqrt{2d\sigma Yt}\,\hat z$ and $\sigma=e^{-Y}/(4\pi)$ with
independent $Y\sim\GammaDist(d/2+a+b+1,\ \text{rate } d/2+a-c+1)$,
$t\sim\Beta(d/2+a,b+1)$ and $\hat z$ uniform. Instead of sampling, each factor
can be integrated by a Gauss rule: generalized Gauss--Laguerre in $Y$,
Gauss--Jacobi in $t$, Gauss--Legendre (or Gauss--Jacobi with weight $\rho$
for the source measure) in $\rho$, and, for the direction, the trapezoidal
rule on the circle when $d=2$ or a randomly rotated cross-polytope
$\{\pm Qe_i\}$ when $d>2$ (exact for polynomials of degree three on the
sphere; the random rotation $Q$ keeps the estimate unbiased). With a
direction set symmetric under $\hat z\mapsto-\hat z$, the Watson and source
integrands are even in $|z|$ and hence smooth functions of $Y$ and $t$; the
integrand of~\eqref{eq:score-ibp} is odd, $|z|$ times a smooth function, and
the factor $\sqrt{Yt}$ is moved into the Gauss weights (exponents raised by
$\frac12$). For integrands with sufficient finite regularity the scalar Gaussian
rules admit regularity-dependent algebraic error bounds. Exponential or
geometric convergence requires stronger analyticity and growth assumptions,
which are not imposed here. The randomly rotated cross-polytope direction
rule is an unbiased randomized angular cubature rule, not a deterministic
Gaussian rule with exponential convergence in dimension. Experiment~E8
measures the observed combined accuracy for the tested smooth examples; it
does not establish a general exponential rate.
\end{remark}

Table~\ref{tab:replacements} summarizes the replacements.

\begin{table}[ht]
\centering
\caption{Network derivatives in the heat-ball residuals and their
replacements. ``Evals'' counts extra network evaluations per source sample.
Accuracy and cost are measured in E7 and E8 (Section~\ref{sec:experiments}).}
\label{tab:replacements}
\small
\begin{tabular}{p{3.0cm}p{4.6cm}cp{4.2cm}}
\toprule
Term & Replacement & Evals & Effect \\
\midrule
$|\nabla v_\theta|^2$ in $h_v$, $h_w$ & $d\,\delta_\varepsilon v_\theta(y;\xi)^2$,
$\xi$ uniform (Prop.~\ref{prop:sqnorm}) & 2 & biased by $O(\varepsilon^2)$;
relative variance $<2$ \\
$\ft\cdot\nabla v_\theta$ & $|\ft|\,\delta_\varepsilon v_\theta(y;\ft/|\ft|)$ & 2 &
deterministic, $O(\varepsilon^2)$ \\
$(\nabla s_\theta)^{\!\top}(2s_\theta-\ft)$ & $|\xi|\,\delta_\varepsilon s_\theta(y;\xi/|\xi|)$,
$\xi=2s_\theta-\ft$ & 2 & deterministic, $O(\varepsilon^2)$; exact for gradient fields \\
$\nabla\cdot\ft$, $\nabla\cdot(\ft u)$, $\nabla\cdot(\ft w)$ & antithetic integration by
parts (Prop.~\ref{prop:ibp}) & 2 & unbiased; no derivative of drift or model \\
whole source $S_r[\nabla h_v]$ of the score identity & integration by parts
(Prop.~\ref{prop:score-ibp}) & 2 & unbiased; values of $s_\theta$, $\ft$, $\nabla\cdot\ft$ only \\
$\nabla(\nabla\cdot\ft)$, $(\nabla\ft)^{\!\top}s$ & none needed: drift only & 0 &
closed form for analytic drifts \\
Score $s(x_0)$ from values of $v$ (evaluation) & $\frac dh\sum_{i\le h}\xi_i\,\delta_\varepsilon v(x_0;\xi_i)$ &
$2h$ & relative error $\sqrt{(d-1)/h}$: poor for large $d$ \\
\bottomrule
\end{tabular}
\end{table}

\begin{remark}[When to use the derivative-free variants]\label{rem:when-fd}
With all replacements, a residual-based loss depends on the network only
through its values, and its parameter gradient is a single backward pass.
One might expect this to be cheaper than differentiating through
$\nabla_xv_\theta$. However, this is not always the case. For multilayer perceptrons of width 256 and 1024 on CPU, the cost depends on the specific network and the specific derivative-free form. For scalar networks, automatic differentiation through the input gradient is about 1.5 times faster and uses less memory than finite differences (which require extra forward evaluations) (Table~\ref{tab:cost}). On the other hand, when applied to vector networks, the IBP derivative-free form actually runs 1.04–1.4 times faster than automatic differentiation (Table~\ref{tab:e10}). Consequently, while the IBP form offers a clear speedup in certain settings, the finite-difference variants are primarily intended for models without efficient second-order automatic differentiation---black-box or non-differentiable density models, numerical solvers, or architectures where double backpropagation is unsupported or memory-bound---and as a check on implementations. In float32, steps
$\varepsilon\in[10^{-2},10^{-1}]$ (relative to the data scale) keep the
relative error of the differences at or below about $10^{-3}$ for the networks in E7.
\end{remark}

\authorcheck{The ICLR submission by the same authors derives a different,
ratio-form, derivative-free estimator of $\nabla\log w$ from values of $w$,
with its own variance analysis. To avoid overlap, this paper uses only the
linear central differences above, which are classical
(cf.\ simultaneous-perturbation and Gaussian-smoothing gradient
estimators~\cite{Spall1992,NesterovSpokoiny2017}). Once the ICLR paper is
public, cite it here and in Table~\ref{tab:replacements}.}

%==========================================================
\section{Numerical experiments}\label{sec:experiments}
%==========================================================

We use closed-form solutions, high-accuracy quadrature, or a
high-resolution finite-volume solver as references. Grid convergence of
the finite-volume solver is reported below. Error statistics are based on
the independent repetitions specified in the captions; training results
use five seeds and are given as mean~$\pm$~standard deviation. The
samplers follow Algorithm~\ref{alg:family}, with antithetic pairs
$(z,-z)$ for Watson means. Script names are provided in the captions.

Experiments E0--E2 check the sampling laws and unbiasedness, E3 tests the
small-radius expansion, and E4--E5 examine the source terms for linear and
nonlinear drift. E6--E8 compare accuracy and cost, while E9 studies the
residuals of learned score models.

\subsection{Exact sampling and unbiasedness}\label{sec:exp-sampling}
The first three experiments check Theorem~\ref{thm:family} and the statements of
Section~\ref{sec:error} that follow from it: that Algorithm~\ref{alg:family}
draws from $\mu_{a,b,c}$ exactly, that the estimators are unbiased with the
variance of Proposition~\ref{prop:variance}, and that importance sampling from a
bounding cylinder loses the factor of Proposition~\ref{prop:variance-gap}.

\paragraph{E0. Verification of the sampler family.}
For eight exponent triples $(a,b,c)$ in $d\in\{1,2,3,5,8\}$, including the
Watson, source, integration-by-parts and uniform members and non-integer
exponents, Table~\ref{tab:e0} compares the means of $|z|^2$, $\sigma$ and
$e^{-|z|^2}(1+\sigma)$ under Algorithm~\ref{alg:family} ($1.5\times10^6$
samples, the same PyTorch sampler used in all later experiments) with direct
two-dimensional quadrature over $\Ecal_1$, and the closed-form mass
$Z_{a,b,c}$ of~\eqref{eq:family-mass} with quadrature. All discrepancies are
within $1.5$ Monte Carlo standard errors, and the closed-form masses agree with
quadrature to the quadrature accuracy.

\begin{table}[ht]
\centering
\caption{E0: sampler family. ``Max.\ rel.\ dev.'' is the largest relative
deviation between Monte Carlo and quadrature over the three test means;
``MC s.e.'' the largest relative standard error of those means. Script
\texttt{e0\_family.py}.}
\label{tab:e0}
\small
\begin{tabular}{lrrrrccc}
\toprule
member & $d$ & $a$ & $b$ & $c$ & max.\ rel.\ dev. & MC s.e. & $Z_{a,b,c}$ (rel.\ dev.\ from quadrature) \\
\midrule
Watson $\mu_\kappa$ & 2 & 1 & 0 & 2 & $1.4\times10^{-3}$ & $9.6\times10^{-4}$ & $4.000000$ \ ($4\times10^{-16}$)\\
source $\mu_\psi$ & 2 & 0 & 1 & 0 & $5.3\times10^{-4}$ & $6.3\times10^{-4}$ & $9.947184\times10^{-3}$ \ ($4\times10^{-16}$)\\
IBP $\mu_{1/2,0,1}$ & 2 & $\tfrac12$ & 0 & 1 & $3.2\times10^{-4}$ & $6.0\times10^{-4}$ & $1.814437\times10^{-1}$ \ ($1\times10^{-11}$)\\
uniform & 2 & 0 & 0 & 0 & $4.4\times10^{-4}$ & $5.3\times10^{-4}$ & $1.989437\times10^{-2}$ \ ($4\times10^{-16}$)\\
mixed & 3 & $\tfrac12$ & 1 & 1 & $6.0\times10^{-4}$ & $6.3\times10^{-4}$ & $1.342870\times10^{-1}$ \ ($2\times10^{-16}$)\\
mixed & 5 & 2 & 2 & $\tfrac32$ & $3.6\times10^{-4}$ & $4.9\times10^{-4}$ & $7.404630\times10^{-2}$ \ ($9\times10^{-10}$)\\
Watson & 1 & 1 & 0 & 2 & $8.3\times10^{-4}$ & $1.5\times10^{-3}$ & $4.000000$ \ ($3\times10^{-10}$)\\
source & 8 & 0 & 1 & 0 & $4.2\times10^{-4}$ & $3.5\times10^{-4}$ & $5.215189\times10^{-3}$ \ ($1\times10^{-10}$)\\
\bottomrule
\end{tabular}
\end{table}

\paragraph{E1. Mean-value identity, unbiasedness and the choice of radius.}
Figure~\ref{fig:e1}(a) shows the error of the Watson mean of the
backward-caloric polynomial $|z|^2-2d\sigma$, whose exact mean is $0$
(Corollary~\ref{cor:constants}), for $d\in\{2,16,128,1024\}$: the estimator is
unbiased and converges at the rate $N^{-1/2}$ in every dimension, with a
constant that grows with $d$ because $\mathrm{Var}(|z|^2-2d\sigma)$ does. Panels (b)--(c) use the heat flow of a
three-component Gaussian mixture (zero drift, $\tau_0=0.3$), for which
$M_r[u]=u(x_0,\tau_0)$ exactly ($200$ repetitions for $d\le16$, $40$
otherwise). With the radius scaled as $r=0.5\sqrt{2/d}$, so that $r\sqrt d$ and hence
the spread of $\log u$ over the samples stay fixed
(Remark~\ref{rem:radius}), the relative error is $N^{-1/2}$ and nearly
independent of $d$ ($2.9\times10^{-3}$ at $d=1024$, $N=4096$), and antithetic
pairs reduce it by a further factor $6$--$15$. At a fixed radius $r=0.5$ the
same estimator is still unbiased but its error grows rapidly with $d$; at
$d=1024$ the values $u(x_0+rz,\cdot)/u(x_0,\tau_0)$ are log-normally
heavy-tailed and $N\le4096$ samples do not reach the asymptotic regime
(panel (c)).
As a direct check of unbiasedness we divide the signed mean error over the
repetitions by its standard error. With the dimension-scaled radius this ratio lies
within $\pm2.5$ in all $22$ cases (all $d$ and $N$), and for the polynomial within $\pm2$ in
all $44$, as expected for a statistic with mean zero. With the fixed radius it lies within
$\pm1.6$ for $d\le128$ but reaches $-5.5$ at $d=1024$. The estimator is still unbiased
there; its distribution is so skewed that the average of $40$ repetitions usually falls
below the mean and the empirical standard error underestimates the true one. This is consistent with the heavy-tail effect of panel (c), and supports
dimension-dependent radius selection for these examples.

\begin{figure}[ht]
\centering
\includegraphics[width=\textwidth]{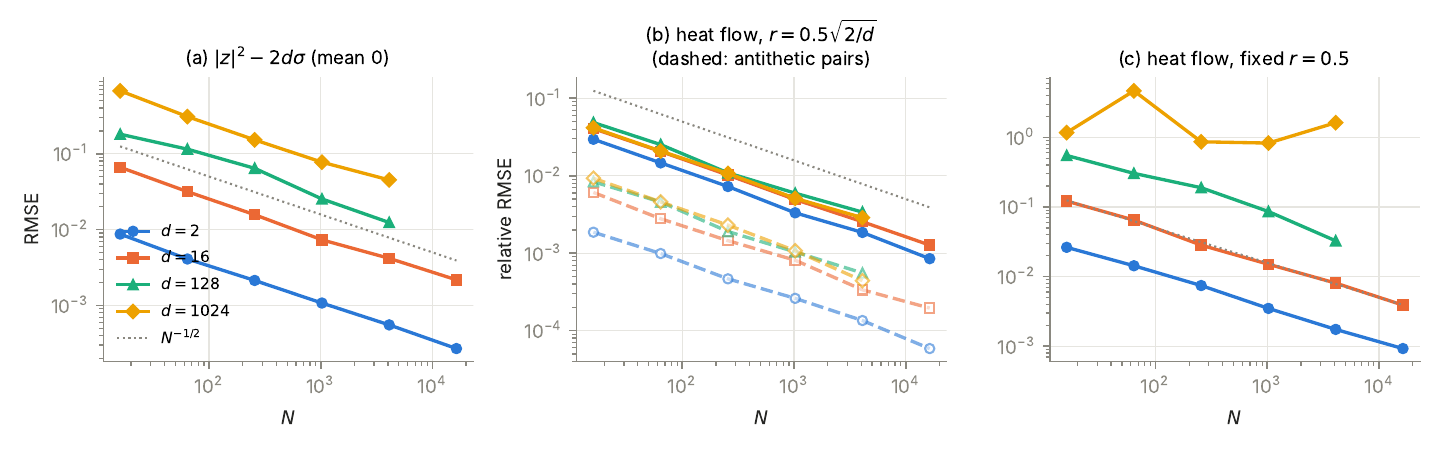}
\caption{E1: error of the Watson mean against $N$. (a) Backward-caloric
polynomial (exact mean $0$). (b) Heat flow of a Gaussian mixture with
dimension-scaled radius; dashed: antithetic pairs. (c) Same with fixed
radius. Dotted: slope $-1/2$. Script \texttt{e1\_unbiased.py}.}
\label{fig:e1}
\end{figure}

\paragraph{E2. Dimension dependence and the bounding cylinder.}
For a constant integrand the efficiency loss of bounding-cylinder importance
sampling is $\chi^2(\mu_\kappa\Vert q)+1=\frac{|S|}{4}\E_{\mu_\kappa}[|z|^2/\sigma^2]$,
and by Theorem~\ref{thm:family},
$\E_{\mu_\kappa}[|z|^2/\sigma^2]=8\pi d\frac{d+2}{d+4}\,\E[Ye^Y]$ with
$Y\sim\GammaDist(d/2+2,\text{rate }d/2)$, i.e.
$\E[Ye^Y]=a\lambda^a(\lambda-1)^{-a-1}$ with $a=d/2+2$, $\lambda=d/2$. This is
finite only for $d>2$: \emph{for $d\le2$ the bounding-cylinder estimator has
infinite variance(therefore, the empirical RMSE values reported for $d=2$ in Table 7 are finite-sample observations and do not represent stable theoretical errors)}. Figure~\ref{fig:e2}(a) shows the closed form, its Monte
Carlo estimate and the lower bound $|S|/|\Ecal_1|$ of
Proposition~\ref{prop:variance-gap}: $\chi^2+1=108$, $26.0$, $7.32$, $5.46$,
$7.05$, $12.7$, $24.75$ for $d=3,4,8,16,64,256,1024$, approaching the bound
($24.56$ at $d=1024$) and growing like $\sqrt d$. (b) The variance of the
Watson term divided by its leading-order prediction
$r^2|\nabla F|^2C_1/2$ (Proposition~\ref{prop:variance}) is $0.993$--$1.005$
for $d\in\{2,16,256\}$ and $r\in\{0.4,0.1,0.025\}$. (c)
Figure~\ref{fig:e2}(b) extends Table~\ref{tab:semigroup} to $d=1024$: at matched
mean time lag the heat-ball and Gaussian-semigroup estimators have the same
error to within a factor $1.4$ at every dimension and radius
($30$ repetitions, $N=4096$),whereas the bounding-cylinder estimator is consistently inferior across all dimensions.

\begin{figure}[ht]
\centering
\includegraphics[width=0.92\textwidth]{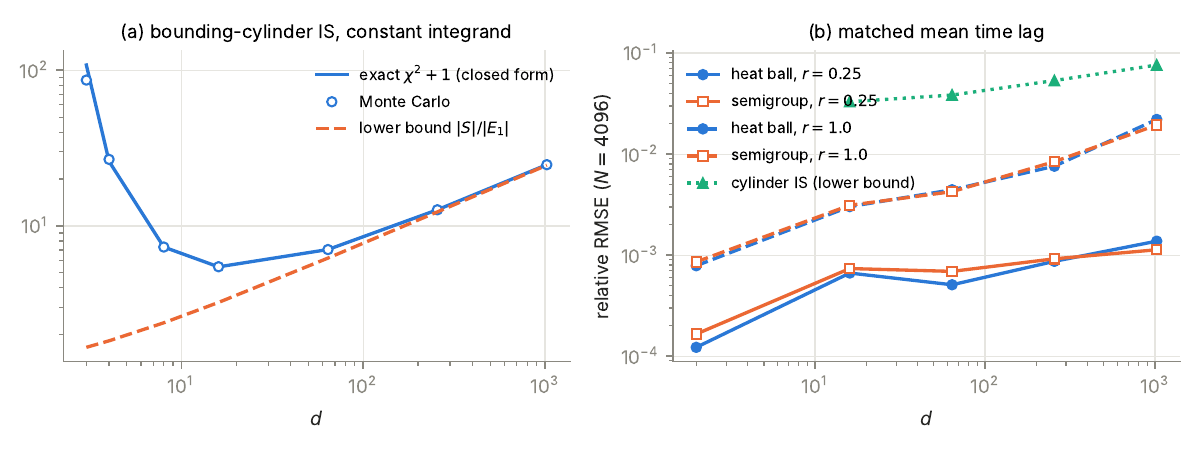}
\caption{E2. (a) $\chi^2(\mu_\kappa\Vert q)+1$ of bounding-cylinder importance
sampling: closed form, Monte Carlo ($4\times10^7/d$ samples, capped at
$4\times10^5$) and the lower bound of Proposition~\ref{prop:variance-gap};
infinite for $d\le2$. (b) Relative RMSE for the heat flow of a Gaussian
($|\mu|=1.5$, initial variance $0.25$, $\tau_0=0.5$): heat ball and Gaussian
semigroup at matched mean time lag, and the constant-integrand lower bound
for the bounding cylinder. Script \texttt{e2\_dimension.py}.}
\label{fig:e2}
\end{figure}

E0--E2 are consistent with the sampling laws and the theoretical variance
behavior in the dimensions tested, including $d=1024$ for selected cases. They also show two limits. Unbiasedness does not make the density
estimator usable at a fixed radius in high dimension (E1(c)); the radius has to
shrink like $d^{-1/2}$. And the heat ball matches the Gaussian semigroup at
matched time lag rather than improving on it.

\subsection{The small-radius limit}\label{sec:exp-radius}
Theorem~\ref{thm:dsm-limit}(i) states that the heat-ball residual of a candidate
solution is $C_1r^2/4$ times its pointwise Fokker--Planck residual, up to higher
order in $r$. E3 checks the rate and the constant.

\paragraph{E3. Small-radius rate.}
For the exact log-density $v$ of a three-component Gaussian mixture under an
anisotropic Ornstein--Uhlenbeck process and the candidate
$F=v+0.05\sin(k\cdot x)e^{-2\tau}$, the residual $\mathcal R_r[F;h_v]$ is
computed as $\mathcal R_r[F]-\mathcal R_r[v]$ with common random numbers
($5\times2\times10^5$ samples per radius; $\mathcal R_r[v]=0$ exactly).
Figure~\ref{fig:e3} shows slope $2$ over two decades of $r$ and convergence of
$4\mathcal R_r/(C_1r^2)$ to the pointwise residual ($-0.0145$ for $d=2$,
$-0.0017$ for $d=16$), as stated in Theorem~\ref{thm:dsm-limit}(i); the
scaled residual is within $3\%$ of its limit for $r\le0.4$ and departs from it
only when $r$ approaches the admissible maximum $\sqrt{4\pi\tau_0}$.

\begin{figure}[ht]
\centering
\includegraphics[width=0.92\textwidth]{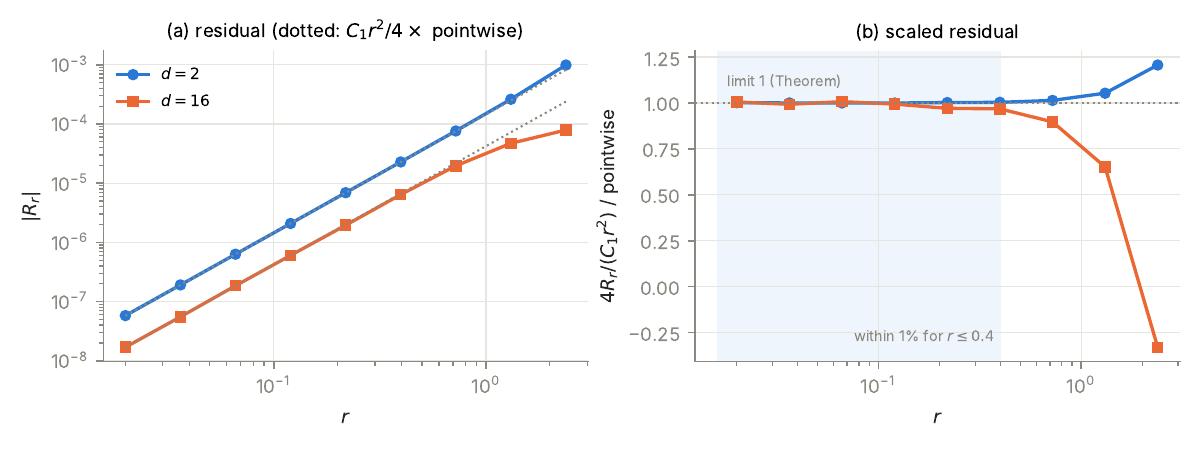}
\caption{E3: (a) $|\mathcal R_r|$ against $r$; dotted: $C_1r^2/4$ times the
pointwise residual. (b) Scaled residual divided by the pointwise residual; the
limit $1$ is that of Theorem~\ref{thm:dsm-limit}(i), and in the shaded range $r\le0.4$ the
scaled residual is within $1\%$ of it.
Script \texttt{e3\_small\_radius.py}.}
\label{fig:e3}
\end{figure}

The rate and the coefficient agree numerically for one smooth perturbation in $d=2$ and
$d=16$. We have not tested non-smooth candidates, for which the expansion of
Theorem~\ref{thm:dsm-limit} need not hold.

\subsection{Source terms under general drift}\label{sec:exp-source}
The identities keep the source terms exactly, including those of the nonlinear
quantities. E4 tests this for a linear drift with rotation, where the solution is
known in closed form and every source is nonzero; E5 does the same for nonlinear
drifts, with a finite-volume reference, and asks how much is lost when the source
is dropped.

\paragraph{E4. Linear drift with exact solution and nonzero source.}
The forward process is $\dd x=Bx\,\dd\tau+\sqrt2\,\dd W$ with
$B=\mathrm{blockdiag}\bigl(\begin{smallmatrix}-0.5&-1\\1&-1.5\end{smallmatrix}\bigr)$
(anisotropic, with a rotation), started from a three-component Gaussian
mixture (means of norm $2$, standard deviations $0.4$--$0.6$). The marginal is
a Gaussian mixture in closed form, and $h_u$, $h_v$, $\nabla h_v$ are
nonzero. At $\tau_0=0.5$ we evaluate $u$ by the completely derivative-free
density identity (Corollary~\ref{cor:density-free}), $\log u$
by~\eqref{eq:estimator} and the score by~\eqref{eq:score-mean}, with sources
computed from the exact solution, at a high-density anchor (the mode) and a
low-density anchor (the $5\%$ quantile of $\log p_{\tau_0}$).
Table~\ref{tab:e4} shows that the log-density and score identities are
accurate to $10^{-6}$--$10^{-3}$ in every case including $d=128$, while the
density identity degrades with $d$ and $r$ for the reason seen in E1: on a
heat ball of fixed radius $u$ varies over orders of magnitude when $d$ is
large. This supports evaluating $\log u$ and $\nabla\log u$ rather than $u$ in
high dimension.

\begin{table}[ht]
\centering
\caption{E4: mean relative error of $u$, $\log u$ and $\nabla\log u$ from the
heat-ball identities with nonzero source; $N=N'=4096$ (in parentheses:
$N=N'=256$); $10$ repetitions ($4$ for $d=128$). Script
\texttt{e4\_linear\_exact.py}.}
\label{tab:e4}
\small
\begin{tabular}{rlcccc}
\toprule
$d$ & anchor & $r$ & $u$ & $\log u$ & $\nabla\log u$ \\
\midrule
2 & high & 0.3 & $3.5\times10^{-5}$ ($1.6\times10^{-4}$) & $1.8\times10^{-5}$ & $7.9\times10^{-4}$ ($7.6\times10^{-3}$)\\
2 & low  & 0.3 & $8.7\times10^{-5}$ ($5.0\times10^{-4}$) & $1.2\times10^{-5}$ & $3.2\times10^{-5}$ ($1.7\times10^{-4}$)\\
2 & low  & 1.0 & $1.8\times10^{-3}$ ($5.6\times10^{-3}$) & $1.4\times10^{-4}$ & $4.4\times10^{-4}$ ($1.5\times10^{-3}$)\\
16 & high & 0.3 & $4.0\times10^{-4}$ ($1.4\times10^{-3}$) & $1.1\times10^{-5}$ & $2.5\times10^{-5}$ ($8.0\times10^{-5}$)\\
16 & low  & 0.3 & $2.5\times10^{-3}$ ($8.7\times10^{-3}$) & $4.6\times10^{-6}$ & $1.7\times10^{-5}$ ($5.6\times10^{-5}$)\\
16 & low  & 1.0 & $3.8\times10^{-2}$ ($1.6\times10^{-1}$) & $7.2\times10^{-5}$ & $1.9\times10^{-4}$ ($9.0\times10^{-4}$)\\
128 & high & 0.3 & $5.9\times10^{-3}$ ($2.8\times10^{-2}$) & $1.5\times10^{-6}$ & $7.0\times10^{-6}$ ($2.5\times10^{-5}$)\\
128 & low  & 0.3 & $8.3\times10^{-3}$ ($7.5\times10^{-2}$) & $2.0\times10^{-6}$ & $6.9\times10^{-6}$ ($2.1\times10^{-5}$)\\
128 & low  & 1.0 & $0.72$ ($0.56$) & $3.0\times10^{-5}$ & $1.1\times10^{-4}$ ($3.8\times10^{-4}$)\\
\bottomrule
\end{tabular}
\end{table}

\paragraph{E5. Nonlinear drift.}
\emph{(a) One-dimensional double well.} The process is
$\dd X=(X-X^3)\,\dd t+\dd W$, i.e.\ $\ft(x)=2(x-x^3)$ in the time
$\tau=t/2$, started from the stated initial law
$p_0=0.6\,\mathcal N(-1,0.25^2)+0.4\,\mathcal N(1.2,0.4^2)$ (not the stationary
law, so the reference is not $2x-2x^3$). The reference is a
Scharfetter--Gummel finite-volume solution on $[-4,4]$ with $3000$ cells and
Crank--Nicolson time stepping ($\Delta\tau=10^{-3}$); halving the number of
cells changes the score by $2.4\times10^{-5}$ in relative norm. The score is
evaluated by Corollary~\ref{cor:score-mean} at $40$ anchors drawn from
$p_{\tau_0}$, with $s$, $h_v$ and $\nabla h_v$ supplied by finite differences of
the grid solution ($65536$ samples per term). The error is reported as the
$p_{\tau_0}$-weighted relative RMSE
$(\E_{p_{\tau_0}}|\hat s-s|^2/\E_{p_{\tau_0}}|s|^2)^{1/2}$.
\emph{(b) Two dimensions.} The drift is a double well with a rotation,
$\ft(x)=-(|x|^2/4-1)x+1.5\,Jx/(1+|x|^2/4)$, $J$ the rotation by $\pi/2$, from
the \texttt{asymmetric\_gmm} initial law, with a $400^2$ grid on $[-6,6]^2$
($200^2$ for the convergence check). Table~\ref{tab:e5} shows that the
identity reproduces the analytic score to $5\times10^{-5}$--$3\times10^{-4}$, at the level
of the grid and Monte Carlo errors, and that \emph{omitting the source term}
increases the error by a factor $40$--$160$: for nonlinear drift $h_v$ is not
negligible even at moderate radius.

\begin{table}[ht]
\centering
\caption{E5(a,b): $p_{\tau_0}$-weighted relative RMSE of the score from
Corollary~\ref{cor:score-mean} with numerically supplied sources, with and
without the source term; residual of the log-density identity
(mean~$\pm$~s.d.\ over anchors); grid-convergence difference of the reference
score. Script \texttt{e5\_nonlinear.py}.}
\label{tab:e5}
\small
\begin{tabular}{lcccccc}
\toprule
problem & $\tau_0$ & $r$ & with source & without source & $\log u$ residual & grid diff. \\
\midrule
1D double well & 0.1 & 0.5 & $4.6\times10^{-5}$ & $5.3\times10^{-3}$ & $(-0.3\pm6.8)\times10^{-5}$ & $2.4\times10^{-5}$\\
1D double well & 0.5 & 0.5 & $6.6\times10^{-5}$ & $7.0\times10^{-3}$ & $(-1.7\pm7.7)\times10^{-5}$ & $2.5\times10^{-5}$\\
1D double well & 0.5 & 1.0 & $1.6\times10^{-4}$ & $2.5\times10^{-2}$ & $(-1.8\pm18)\times10^{-5}$ & $2.3\times10^{-5}$\\
2D well + rotation & 0.2 & 0.5 & $2.6\times10^{-4}$ & $1.1\times10^{-2}$ & $(-7.1\pm24)\times10^{-5}$ & $3.4\times10^{-3}$\\
2D well + rotation & 0.5 & 0.5 & $9.5\times10^{-5}$ & $5.7\times10^{-3}$ & $(0.3\pm14)\times10^{-5}$ & $1.5\times10^{-3}$\\
\bottomrule
\end{tabular}
\end{table}

With the sources kept, the log-density and score identities stay accurate under
rotating and nonlinear drift, and dropping the source is not a harmless
approximation. The density identity does not share this accuracy in high
dimension. In both experiments the sources are computed from the reference
solution, so the tests concern the identities and not a solver: they do not say
how well an unknown $u$ could be recovered from them.

\subsection{Cost and accuracy}\label{sec:exp-cost}
E6 compares the heat-ball estimator with bounding-cylinder importance
sampling and a Gaussian-semigroup estimator at comparable cost. E7 measures
the variance and runtime of the exact integration-by-parts formulas and
the approximate finite-difference replacements in
Section~\ref{sec:derivfree}. E8 compares Monte Carlo with the product
quadrature of Remark~\ref{rem:quadrature}.

\paragraph{E6. Cost and accuracy against alternatives.}
\emph{(a) Free space.} On the E4 problem we estimate $u(x_0,\tau_0)$ with all
estimators derivative-free: the heat ball (antithetic Watson mean plus the
antithetic integration-by-parts source, Proposition~\ref{prop:ibp}); the
Gaussian semigroup with Duhamel source at matched mean time lag
$\ell=C_1r^2/4$, its divergence estimated by the antithetic Stein identity
$\E\,\nabla\cdot G(x+aZ)=\E[(G(x+aZ)-G(x-aZ))\cdot Z]/(2a)$; and the
bounding-cylinder estimator for the Watson part (with the heat-ball source
term). Table~\ref{tab:e6} reports relative RMSE and the work-normalized error
$\mathrm{RMSE}^2\times$time. The heat ball and the semigroup are within a factor
$2$ of each other in RMSE in every case, in either direction; the bounding
cylinder is worse by factors from $2.5$ up to $3.7\times10^{3}$, most in low
dimension where its variance is large or infinite (E2). This confirms
Remark~\ref{rem:semigroup}: in free space the heat ball is not a
variance-reduction device.

\begin{table}[ht]
\centering
\caption{E6(a): relative RMSE of $u(x_0,\tau_0)$ on the E4 problem
($N=N'=8192$ for the heat ball; the same number of evaluations for the
semigroup; $30$ repetitions, $10$ for $d=128$) and, in parentheses, the
work-normalized error $\mathrm{RMSE}^2\times$seconds. Script
\texttt{e6\_cost.py}.}
\label{tab:e6}
\small
\begin{tabular}{rlcccc}
\toprule
$d$ & anchor & $r$ & heat ball & semigroup & bounding cylinder \\
\midrule
2 & high & 0.3 & $3.0\times10^{-5}$ ($1.0\times10^{-10}$) & $3.8\times10^{-5}$ ($2.1\times10^{-10}$) & $0.11$ ($1.3\times10^{-3}$)\\
2 & low & 1.0 & $1.4\times10^{-3}$ ($2.2\times10^{-7}$) & $1.1\times10^{-3}$ ($1.6\times10^{-7}$) & $0.33$ ($1.1\times10^{-2}$)\\
16 & high & 0.3 & $2.6\times10^{-4}$ ($5.5\times10^{-8}$) & $2.8\times10^{-4}$ ($8.3\times10^{-8}$) & $2.3\times10^{-2}$ ($3.5\times10^{-4}$)\\
16 & low & 1.0 & $2.6\times10^{-2}$ ($5.7\times10^{-4}$) & $1.4\times10^{-2}$ ($1.8\times10^{-4}$) & $6.6\times10^{-2}$ ($2.4\times10^{-3}$)\\
128 & high & 0.3 & $9.2\times10^{-3}$ ($6.9\times10^{-4}$) & $8.3\times10^{-3}$ ($7.5\times10^{-4}$) & $4.8\times10^{-2}$ ($1.3\times10^{-2}$)\\
128 & low & 0.3 & $1.6\times10^{-2}$ ($2.1\times10^{-3}$) & $1.0\times10^{-2}$ ($1.1\times10^{-3}$) & $5.0\times10^{-2}$ ($1.4\times10^{-2}$)\\
\bottomrule
\end{tabular}
\end{table}

\emph{(b) Bounded domain.} On the half-space $\{x_1>0\}$ with homogeneous
Dirichlet data, $u$ is the heat kernel from $y=(0.5,0,\dots)$ minus its image
(started at $\tau=-0.05$), and the anchor is at distance $0.3$ from the
boundary, $\tau_0=0.1$. For $r<0.3/R_{\max}(d)$ the heat ball lies in the
domain and the Watson mean uses interior values only. A free-space semigroup
estimator must extend $u$ (here by $0$). At matched \emph{mean} time lag the
semigroup rarely leaves the domain and its bias is within the noise; at
matched \emph{maximal} time depth $r^2/(4\pi)$ its bias is $1.2\%$ and $16\%$
for $r=0.5$ and $0.95$ of the admissible radius in $d=2$, and $0.13\%$ in
$d=16$, while the heat-ball estimator is unbiased within its standard error
($\le1.2\times10^{-4}$) in every case. The advantage of the heat ball here is
structural rather than statistical: exactness at finite radius with no
boundary treatment.

\paragraph{E7. Derivative-free variants: accuracy and cost.}
This experiment checks every replacement of Table~\ref{tab:replacements}
against an exact solution and measures its cost.

\emph{Test problem.} The forward process is the anisotropic
Ornstein--Uhlenbeck process $\dd x=-Ax\,\dd\tau+\sqrt2\,\dd W$ with
$A=\mathrm{diag}(a_i)$, $a_i$ equally spaced in $[0.5,2]$, started from a
four-component Gaussian mixture (means of norm $2$, standard deviations
$0.3$--$0.6$, weights $0.4,0.3,0.2,0.1$). The marginal is a Gaussian mixture
with diagonal covariances, so $v=\log u$, the score and Hessian--vector
products are exact; a finite-difference check of the PDE confirms them
(relative residual $<10^{-5}$ up to $d=256$). Anchors are at $\tau_0=0.5$,
$r=0.5$, one near a mixture mean (``high'') and one $2.5$ component standard
deviations away (``low'').

\emph{Truncation error.} With coordinate directions (no sampling noise), the
relative errors of the squared norm, the transport difference and the
Jacobian--vector difference at $d=16$ are $2.0\times10^{-4}$,
$4.1\times10^{-4}$ and $7.3\times10^{-5}$ at $\varepsilon=0.1$, and fall by a
factor $100$ per decade of $\varepsilon$ ($2.0\times10^{-8}$,
$4.1\times10^{-8}$, $7.3\times10^{-9}$ at $\varepsilon=10^{-3}$), confirming
the $O(\varepsilon^2)$ rate; at $d=256$ they are smaller still. In float32 on
a random multilayer perceptron, rounding dominates below
$\varepsilon\approx3\times10^{-3}$: the relative error of the mean squared
directional derivative is $2\times10^{-5}$ at $\varepsilon=0.1$,
$8\times10^{-5}$ at $10^{-2}$ and $9\times10^{-3}$ at $10^{-4}$.

\emph{Variance of the squared-norm estimator.} The measured relative variance
per direction is $0.500$, $1.661$, $1.918$, $2.004$ for $d=2,16,64,256$,
against $2(d-1)/(d+2)=0.500$, $1.667$, $1.909$, $1.977$
(Proposition~\ref{prop:sqnorm}; $2\times10^5$ directions). At equal cost ($32$
directions), averaging $32$ squared differences gives relative variance
$0.016$--$0.063$, whereas the product of two independent $16$-direction
gradient estimates gives $0.065$, $0.27$, $0.48$, $1.23$; the gradient
estimates themselves have relative error $0.25$, $0.97$, $1.98$, $3.99$,
matching $\sqrt{(d-1)/16}$.

\emph{Residual of the true log-density.} Table~\ref{tab:fd-residual} compares
the residual~\eqref{eq:res-scalar} with exact derivatives and with every
derivative replaced (squared difference, transport difference, antithetic
integration by parts for $\nabla\cdot\ft$; $\varepsilon=10^{-3}$, float64). Both
vanish within about two standard errors in every case. The replacements
leave the variance unchanged at $d=2$ and inflate it by $1.1$--$1.7$ at $d=16$
and $4$--$13$ at $d=256$, almost entirely through the squared-norm term,
whose noise is proportional to $|\nabla v|^2$ and hence largest at
low-density anchors. Using $k$ directions per source sample divides that
contribution by $k$.

\begin{table}[ht]
\centering
\caption{Residual $\mathcal R_r[v]$ of the exact log-density (mean
$\pm$ standard error) with exact derivatives and with all derivatives
replaced by values, and the ratio of their per-sample variances. $S_r$ is the
size of the source term. $2\times10^5$ samples per term ($6\times10^5$ for
$d=256$). Script \texttt{accuracy2.py}.}
\label{tab:fd-residual}
\small
\begin{tabular}{rlrrrr}
\toprule
$d$ & anchor & $S_r$ & exact derivatives & derivative-free & variance ratio \\
\midrule
2   & high & 0.0063 & $(-2\pm4)\times10^{-5}$ & $(1\pm4)\times10^{-5}$ & 1.11 \\
2   & low  & 0.0069 & $(-1.7\pm1.3)\times10^{-4}$ & $(-1.7\pm1.3)\times10^{-4}$ & 1.01 \\
16  & high & 0.144 & $-0.0011\pm0.0010$ & $-0.0017\pm0.0010$ & 1.11 \\
16  & low  & 0.243 & $0.0057\pm0.0030$ & $0.0061\pm0.0039$ & 1.73 \\
64  & high & 0.732 & $0.0034\pm0.0025$ & $0.0019\pm0.0032$ & 1.68 \\
64  & low  & 1.176 & $0.0040\pm0.0059$ & $0.0089\pm0.0124$ & 4.39 \\
256 & high & 3.055 & $0.0009\pm0.0032$ & $0.0037\pm0.0066$ & 4.3 \\
256 & low  & 4.531 & $-0.0015\pm0.0065$ & $-0.0180\pm0.0237$ & 13.3 \\
\bottomrule
\end{tabular}
\end{table}

\emph{Sensitivity.} For the perturbed candidate $v+0.05\sin(k\cdot x)$ in
$d=16$ (common random numbers), the exact-derivative and derivative-free
residuals agree to within $1\%$ at $r=0.4,0.2,0.1$ (e.g.\ scaled residuals
$0.0310$ and $0.0308$ at $r=0.4$), so the replacements detect a wrong
candidate exactly as the original residual does.

\emph{Density and vector residuals.} The completely derivative-free density
residual (Corollary~\ref{cor:density-free}) of the exact solution is
$(-0.6\pm0.8)\times10^{-4}$, $(-6.4\pm3.6)\times10^{-4}$ and
$(-1.5\pm4.1)\times10^{-4}$ relative to $u(x_0,\tau_0)$ for $d=2,8,16$
(source sizes $0.0017$, $0.014$, $0.063$); the antithetic pairing reduces the
standard deviation of the integration-by-parts term by factors $57$, $4.8$
and $3.6$. In higher dimension $u=e^v$ varies over many orders of magnitude on
the heat ball and the log-density residual is preferable. The vector residual
with the Jacobian--vector difference agrees with the exact one to all printed
digits for $d=2,16,64$, with relative residual
$(3.0\pm2.3)\times10^{-4}$, $(3.4\pm3.5)\times10^{-4}$, $(5.9\pm4.9)\times10^{-4}$.

\emph{Cost.} Table~\ref{tab:cost} reports the time and memory of one training
step (loss and parameter gradient) for a multilayer perceptron of width $256$
and three hidden layers, with $32$ anchors and $16+16$ samples per anchor
($1024$ evaluation points), in float32 on a 2-core CPU with JAX. The
pointwise residuals use the same $1024$ collocation points. The heat-ball
residual is $4.6$--$5$ times cheaper than the pointwise residual with a
one-probe Hutchinson Laplacian, and $42$ and $200$ times cheaper than with the
exact Laplacian at $d=16$ and $64$; the exact Laplacian could not be run at
$d=256$ within $7$\,GB. The derivative-free variant of the log-density
residual is about $1.5$ times slower than automatic differentiation; the
completely derivative-free density residual costs the same as automatic
differentiation; for the vector model the Jacobian--vector difference is
equal in cost at $d\le64$ and up to $1.5$ times slower beyond. For a wider
network (width $1024$, six hidden layers, $d=256$) the step takes $307$\,ms
with automatic differentiation and $454$\,ms with differences.

\begin{table}[ht]
\centering
\caption{Time (ms) and compiled memory (MB) of one training step. HB-AD and
HB-FD: log-density heat-ball residual with automatic differentiation and with
all derivatives replaced; HB-$u$: completely derivative-free density residual;
VEC-AD and VEC-FD: vector residual with a vector--Jacobian product and with a
Jacobian--vector difference; PW: pointwise residual with Hutchinson's
estimator and with the exact Laplacian. Script \texttt{bench\_jax.py}.}
\label{tab:cost}
\small
\begin{tabular}{rccccccc}
\toprule
$d$ & HB-AD & HB-FD & HB-$u$ & VEC-AD & VEC-FD & PW-Hutch. & PW-exact \\
\midrule
16   & 12 / 15 & 18 / 23 & 12 / 15 & 13 / 16 & 13 / 20 & 62 / 48 & 511 / 550 \\
64   & 12 / 15 & 18 / 24 & 10 / 15 & 15 / 17 & 15 / 21 & 58 / 48 & 2419 / 2087 \\
256  & 15 / 17 & 21 / 26 & 15 / 18 & 18 / 20 & 22 / 26 & 69 / 50 & --- \\
1024 & 24 / 25 & 38 / 34 & 23 / 26 & 41 / 34 & 60 / 46 & 111 / 62 & --- \\
\bottomrule
\end{tabular}
\end{table}

\paragraph{E8. The score identity without derivatives, and quadrature instead of sampling.}
On the problem of E4 (anisotropic OU with rotation, three-component Gaussian mixture, exact
score), $\tau_0=0.5$, a high- and a low-density anchor and $r\in\{0.3,1\}$, we evaluate the
score identity in the form~\eqref{eq:res-vector} with a Hessian--vector product (VJP form)
and in the form~\eqref{eq:score-mean-ibp} with values of $s$ only (IBP form), and the
log-density identity, by Monte Carlo with the exact samplers and by the product quadrature of
Remark~\ref{rem:quadrature} (Table~\ref{tab:e10}, Figure~\ref{fig:e10}).

\emph{(a) Monte Carlo.} The IBP form is exact (Proposition~\ref{prop:score-ibp}), but its
variance grows with the dimension: at $N=N'=4096$ its error is $1.3$, $7$ and $21$ times
that of the VJP form for $d=2$, $16$, $64$. Each sample of~\eqref{eq:score-mean-ibp} informs the $d$-vector $s$ only along one
direction $\hat z$, through a difference of values, while a VJP sample carries a whole
gradient; removing the derivative costs a factor of order $d$ in variance, as for the
evaluation of a score from values of $v$ (Table~\ref{tab:replacements}).

\emph{(b) Quadrature.} In two dimensions the product Gauss rule (eight equispaced, symmetric
directions) exhibits rapid convergence over the tested resolutions: with $7\times10^4$ evaluations the log-density error is
$4\times10^{-13}$--$2\times10^{-9}$ and the score error $7\times10^{-13}$--$6\times10^{-7}$, where Monte
Carlo with the exact samplers and the same number of evaluations stays at
$7\times10^{-6}$--$9\times10^{-5}$ and $2\times10^{-5}$--$2\times10^{-2}$. In $d=16$ and $64$ the
randomly rotated cross-polytopes reduce directional integration error in several tested cases; the product rule is substantially more accurate in some comparisons,
though the gain is not uniform across dimensions, anchors, and quantities.
For example, Table~\ref{tab:e10} shows that at $d=64$ the product-rule
score error ($1.2\times10^{-4}$) slightly exceeds the VJP Monte Carlo
error ($1.0\times10^{-4}$) in its reported aggregate. With
$4608$ ($d=16$) or $18\,432$ ($d=64$) evaluations of $s$ and no derivative it is smaller than
that of the VJP form with $8192$ evaluations and $4096$ Hessian--vector products in seven of
the eight cases (anchor $\times$ radius). There the error no longer decreases with the number of
radial nodes and only slowly with the number of rotations: the direction integral is what
remains of the dimension.

\emph{(c) Cost.} For a vector network (MLP, width $256$) the IBP form is $1.04$--$1.4$ times
faster per training step on a CPU ($d=16$ to $1024$), and for a scalar network used through
its gradient $1.1$--$1.2$ times. Removing the derivatives is therefore cheap but not, by
itself, a large saving; its value is that the score identity can be evaluated, by quadrature,
more accurately than with derivatives, and for models without efficient second-order
automatic differentiation.

\begin{table}[ht]
\centering
\caption{E8: relative error of the score identity (mean over two anchors and $r\in\{0.3,1\}$)
by Monte Carlo in the VJP form ($8192$ evaluations of $s$ and $4096$ Hessian--vector products)
and in the IBP form ($12\,288$ evaluations, no derivative), and by product quadrature in the
IBP form; time per training step of the vector residual for a vector network (CPU, $64$
anchors, $N=N'=16$). Script \texttt{e10\_ibp\_score.py}.}
\label{tab:e10}
\small
% Inlined from table_e10.tex (contents checked against the compiled PDF).
\begin{tabular}{rcc cc cc}
\toprule
 & \multicolumn{2}{c}{Monte Carlo, $N=N'=4096$} & \multicolumn{2}{c}{quadrature, form (IBP)}
 & \multicolumn{2}{c}{ms/step, vector net} \\
$d$ & form (VJP) & form (IBP) & evaluations & error & VJP & IBP \\
\midrule
2    & $5.1\times10^{-3}$ & $6.5\times10^{-3}$ & 8704  & $4.1\times10^{-5}$ & --- & --- \\
16   & $1.6\times10^{-4}$ & $1.1\times10^{-3}$ & 4608  & $8.6\times10^{-5}$ & 69  & 61 \\
64   & $1.0\times10^{-4}$ & $2.1\times10^{-3}$ & 18432 & $1.2\times10^{-4}$ & 80  & 77 \\
256  & --- & --- & --- & --- & 120 & 91 \\
1024 & --- & --- & --- & --- & 374 & 261 \\
\bottomrule
\end{tabular}
\end{table}

\begin{figure}[ht]
\centering
\includegraphics[width=\textwidth]{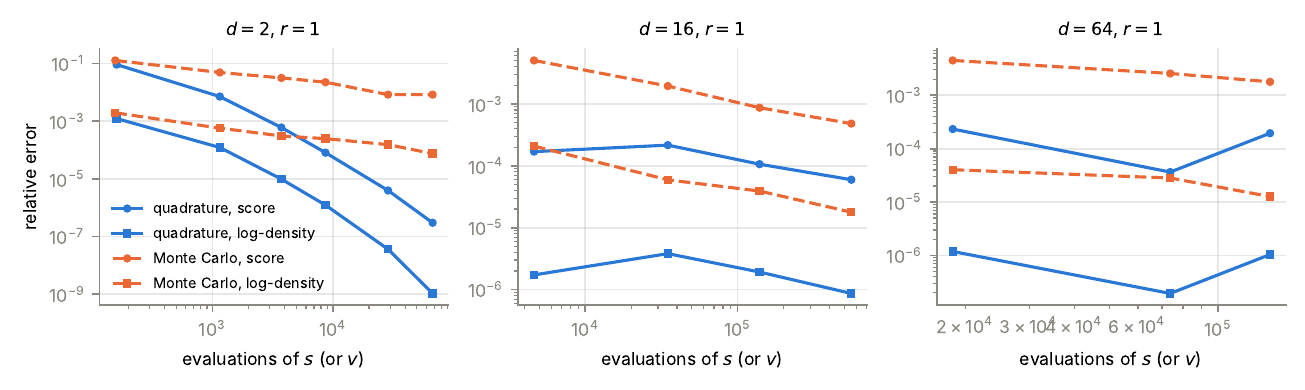}
\caption{E8: relative error of the score and log-density identities against the number of
evaluations, product quadrature (solid) and Monte Carlo with the exact samplers (dashed),
$r=1$, mean over two anchors. For $d=2$ the quadrature is deterministic; for $d=16,64$ the
direction rule is a randomly rotated cross-polytope; the latter curves do not establish a uniform exponential convergence rate.}
\label{fig:e10}
\end{figure}

% GPU timing: produced by  python experiments/run_gpu_all.py --only bench  (e9_results.tex)
\IfFileExists{e9_results.tex}{\input{e9_results.tex}}{}

The cost results are mixed, and we report them as such. At equal cost the heat
ball is as accurate as the semigroup, not more; its advantage in E6(b) comes from
exactness at finite radius, not from variance. Replacing derivatives by integration by parts is
exact, but with Monte Carlo it can cost variance---a factor of order $d$ for the
score---and, on our CPU, some run time. Product quadrature is where the structure
of Theorem~\ref{thm:family} pays most clearly, by factors of $6$ to $900$ in
accuracy up to $d=64$; beyond that its direction rule, with $2d$ points per radial
node, becomes the limit. All timings are on CPU. On a GPU, or for convolutional
and attention networks, where double backpropagation uses more memory, the
balance between automatic differentiation and extra evaluations may differ; we
have not measured it.

\paragraph{Interpretation of the experimental evidence.}
E0--E8 test sampling laws, finite-radius identities, quadrature accuracy and
residual evaluation when the integrands are available from analytic or
high-accuracy reference solutions. Such tests validate the representations
and estimators, but they are not benchmarks of a stand-alone solver that
recovers an unknown density from initial data alone. Where reported,
CPU timings and relative errors refer to the stated hardware, evaluation
counts and reference functions; they do not imply a uniform speed advantage
over Gaussian-semigroup methods. The numerical tables are retained without
recomputation; the new allocation rule in Proposition~\ref{prop:allocation}
is analytical and does not require additional measurements.

\subsection{Residuals of trained score models}\label{sec:exp-residual}
E9 examines the heat-ball residual as a diagnostic for learned scores.
We compare it with a pointwise residual, estimate its Monte Carlo noise,
and test whether its behavior follows score error during denoising score
matching.

\paragraph{E9. Residual as a diagnostic for score models.}
On the E4 problem ($r=0.3$, $256$ anchors for $d=2$ and $128$ for $d=16$,
$\tau$ log-uniform) we perturb the analytic log-density by
$\varepsilon\sin(k\cdot x)(1+\tau)$ (global) or by a Gaussian bump centred at a
low-density point, and the analytic score by a non-gradient rotational field
(vector residual~\eqref{eq:res-vector}). Figure~\ref{fig:e7} compares the
two-sample estimate of the mean squared heat-ball residual (scaled by
$4/(C_1r^2)$) with the pointwise residual computed with an exact Laplacian.
With $N=N'=2048$ the two agree closely over the displayed perturbation range and
both grow like $\varepsilon^2$; the Monte Carlo noise floor of the heat-ball
estimate (its magnitude at $\varepsilon=0$) is $0.003$--$0.015$ at $N=2048$,
$0.04$--$0.3$ at $N=256$ and $0.3$--$2$ at $N=16$, so the small sample sizes used in training
are too noisy for diagnosing small errors. Per unit cost, the two-sample
heat-ball estimate with $N=256$ has $100$ times smaller variance than the
pointwise residual with a one-probe Hutchinson Laplacian at $d=16$, and $500$
times larger at $d=2$, where the exact Laplacian is cheap and noise-free.

\emph{Time normalization.} For networks trained by denoising score matching
the scalar residual does \emph{not} track the score error (score relative MSE
$0.061\to0.0057$ across the budgets of Table~\ref{tab:e9b}, while the
scalar residual is nonmonotone ($34,19,18,32,522$). The last point is
from a separate training run, and hence these values do not constitute
a single-run learning curve; the pointwise residual exhibits the same
normalization sensitivity). Denoising score matching
determines $\nabla_xv_\theta$ only, so $v_\theta$ carries an arbitrary offset
$c(\tau)$ whose derivative enters the scalar residual
(Remark~\ref{rem:normalization}). For diagnosing score models the scalar
residual must therefore be used either with a normalized model or in its
vector form~\eqref{eq:res-vector}, which is invariant under $c(\tau)$.

\begin{figure}[ht]
\centering
\includegraphics[width=0.92\textwidth]{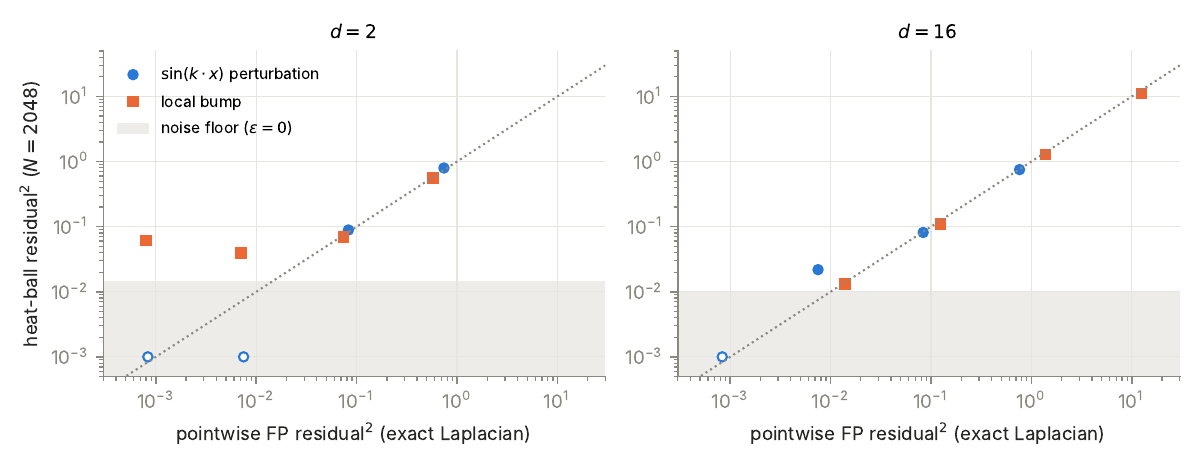}
\caption{E9: mean squared heat-ball residual ($N=N'=2048$, two-sample
estimate) against the pointwise residual with exact Laplacian, for global and
localized perturbations of the exact log-density,
$\varepsilon\in\{0.01,0.03,0.1,0.3\}$. Open markers: negative estimates
(inside the noise floor). Shaded: noise floor at $\varepsilon=0$. Script
\texttt{e7\_diagnostic.py}.}
\label{fig:e7}
\end{figure}

\emph{Vector residual of trained networks.} We repeated the network part with the
vector residual~\eqref{eq:res-vector}, which is invariant under $c(\tau)$, in the VJP
form and in the integrated-by-parts form~\eqref{eq:res-vector-ibp} ($N=N'=2048$, $256$
anchors, $r=0.3$; Table~\ref{tab:e9b}). For the analytic score both forms give $-0.05$ and
$0.47$, which is the noise floor. The two forms agree to within $5\%$ for every network,
as they should, since they estimate the same quantity. Unlike the scalar residual, the
vector residual separates a poorly trained network from well trained ones: it is $76$
after $50$ steps (relative score error $0.061$) and $18$--$30$ once the error is below
$0.02$. Within that range, however, it does not order the networks. After $8000$ steps
the score error is three times smaller than after $3200$, and the residual is larger
($30$ against $18$); at the low-density anchors it rose from $116$ to $436$ between $50$
and $200$ steps while the score error there fell by a factor of three. Over the five
budgets the rank correlations reported for these five budgets are descriptive only
($0.6$ for the vector residual and $-0.1$ for the scalar one),
not estimates of population correlations. This is what one should expect. In terms of the error
$e=s_\theta-s$ the pointwise vector residual is
$\partial_\tau e-\Delta e-\nabla(2s\cdot e+|e|^2-\ft\cdot e)$: it contains derivatives of
$e$, and in low-density regions, where $|s|$ and $|\nabla s|$ are large, the coupling
terms amplify it. A network can reduce $\E|e|^2$ and still make $e$ rougher, and DSM,
whose loss is weighted by $p$, does little to control derivatives of $e$ where data are
scarce. The residual is therefore a test of consistency with the Fokker--Planck equation
and not an estimate of the score error. It detects networks that are clearly wrong, and
it measures something that denoising score matching does not, but a smaller residual
does not imply a smaller score error. (The low-density averages are over $26$ anchors
and a single training run, so the individual numbers there are not precise.)

\begin{table}[ht]
\centering
\caption{E9: networks trained by denoising score matching ($d=2$, problem of E4). Relative
score MSE, mean squared scalar residual and mean squared vector residual (VJP and
integrated-by-parts forms; in parentheses: the $10\%$ lowest-density anchors), two-sample
estimates in pointwise units, $N=N'=2048$, $256$ anchors. analytic score: vector residual
$-0.05$ (VJP) and $0.47$ (IBP). The last row is from a second run of the same protocol and should not be interpreted as the continuation of the other rows.
Script \texttt{e7b\_vector\_trained.py}.}
\label{tab:e9b}
\small
\begin{tabular}{rcccc}
\toprule
DSM steps & score rel.\ MSE (low) & scalar residual & vector, VJP (low) & vector, IBP (low) \\
\midrule
50   & 0.061 (0.076) & 34  & 77 (116) & 76 (114) \\
200  & 0.035 (0.024) & 19  & 73 (436) & 72 (427) \\
800  & 0.017 (0.012) & 18  & 22 (78)  & 21 (72) \\
3200 & 0.018 (0.011) & 32  & 18 (70)  & 19 (66) \\
8000 & 0.0057 (0.0029) & 522 & 30 (176) & 30 (180) \\
\bottomrule
\end{tabular}
\end{table}

E9 supports the heat-ball residual as a consistency check: its estimate matches the
pointwise residual, and in vector form it separates good from clearly bad networks. It
does not support using the residual as a proxy for the score error, and at the sample
sizes used in training its noise floor is too high to detect small errors.

%==========================================================
\section{A further application: heat-ball walks}\label{sec:walks}
%==========================================================

The density identity can also be used recursively. Rather than
supplying the unknown density on a heat ball from a network or a reference
solution, one can continue sampling until initial or boundary data are
reached. This leads to heat-ball walks that use the drift and prescribed
data for pointwise evaluation. We give two-dimensional examples of this
alternative calculation, without undertaking a general convergence
analysis. To construct the walks for nonlinear drift, we first separate
its curl-free and divergence-free parts.

\subsection{Helmholtz decomposition of the drift}\label{sec:helmholtz}

To make the action of the drift explicit, we split the normalized drift by
a Helmholtz (Hodge) decomposition~\cite{Bhatia2013}:
\begin{equation}\label{eq:helmholtz}
  \ft = \ft_1 + \ft_2,\qquad
  \nabla_x\!\cdot\ft_1 = 0\ (\text{divergence-free}),\qquad
  \ft_2 = \nabla_x\phi\ (\text{curl-free}).
\end{equation}
In two dimensions the divergence-free part is a stream-function rotation,
$\ft_1=\nabla^\perp\chi=(-\partial_{x_2}\chi,\partial_{x_1}\chi)$.
On $\R^d$ the split~\eqref{eq:helmholtz} is unique only under decay
conditions at infinity; here it is fixed by construction, through a
parameterization of $\ft$ by a potential $\phi$ and a divergence-free field
(for example a stream function), or by the closed form of the drift.
The source of~\eqref{eq:heat} then separates as
\begin{equation}\label{eq:source-split}
  -\nabla_x\!\cdot(\ft\,u)
  = \underbrace{-\nabla_x\!\cdot(\nabla_x\phi\,u)}_{\text{potential part}}
    \ \underbrace{-\ \ft_1\!\cdot\!\nabla_x u}_{\text{transport part}},
\end{equation}
using $\nabla_x\!\cdot\ft_1=0$. Two objects must be kept apart. The
\emph{drift} need not be curl-free: its divergence-free part is a genuine
transport component of the forward dynamics. The \emph{score}
$\nabla_x\log u$, by contrast, is a gradient and hence curl-free, so any
rotational mass in a decomposed score estimate is an artifact, while
$\|\ft_1\|$ measures how far the forward dynamics depart from a gradient
flow. The mean-value identities of Section~\ref{sec:heatball} treat the drift
through the single term $-\nabla_x\cdot(\ft u)$ and do not need the
split. The walks do: after the gauge transform of
Proposition~\ref{prop:gauge} the curl-free part acts only through a
zero-order potential and the divergence-free part only as a transport, and
Experiment~W3 measures the two parts of a learned drift.

\subsection{A gauge transform separating the two parts of the drift}

\begin{proposition}[Gauge transform]\label{prop:gauge}
Let $\ft=\ft_1+\ft_2$ be the split~\eqref{eq:helmholtz}, with $\ft_2=\nabla\phi=-\nabla\Phi$
(that is, $\Phi:=-\phi$), $\Phi\in C^{2,1}$, and let
$u$ solve $\partial_\tau u=\Delta u-\nabla\cdot(\ft u)$. Then $\varphi:=e^{\Phi/2}u$ solves
\begin{equation}\label{eq:gauge}
  \partial_\tau\varphi=\Delta\varphi-\nabla\cdot(\ft_1\varphi)-V\varphi,\qquad
  V=\tfrac14|\nabla\Phi|^2-\tfrac12\Delta\Phi-\tfrac12\,\ft_1\cdot\nabla\Phi-\tfrac12\partial_\tau\Phi .
\end{equation}
\end{proposition}

\begin{proof}
Write $u=E\varphi$ with $E=e^{-\Phi/2}$, so $\nabla E=-\tfrac12E\nabla\Phi$ and
$\Delta E=E(\tfrac14|\nabla\Phi|^2-\tfrac12\Delta\Phi)$. Then
$\Delta u+\nabla\cdot(u\nabla\Phi)=E\bigl[\Delta\varphi-(\tfrac14|\nabla\Phi|^2-\tfrac12\Delta\Phi)\varphi\bigr]$,
$-\ft_1\cdot\nabla u=E\bigl[-\ft_1\cdot\nabla\varphi+\tfrac12(\ft_1\cdot\nabla\Phi)\varphi\bigr]$
and $\partial_\tau u=E(\partial_\tau\varphi-\tfrac12\varphi\,\partial_\tau\Phi)$. Since
$-\nabla\cdot(\ft u)=\nabla\cdot(u\nabla\Phi)-\ft_1\cdot\nabla u$ and
$\ft_1\cdot\nabla\varphi=\nabla\cdot(\ft_1\varphi)$, dividing by $E$
gives~\eqref{eq:gauge}.
\end{proof}

In the gauge variable the curl-free part of the drift acts only through the zero-order
potential $V$ (absorption where $V>0$, creation where $V<0$) and the divergence-free part
only as a transport in divergence form. For the Ornstein--Uhlenbeck drift
$\ft=-Ax+\omega Jx$ ($A$ diagonal, $J$ the rotation by $\pi/2$),
$V=\tfrac14|Ax|^2-\tfrac12\operatorname{tr}A-\tfrac\omega2x^{\!\top}J^{\!\top}Ax$, a harmonic
potential. For the double well with rotation $\ft=-(|x|^2/4-1)x+\omega Jx/(1+|x|^2/4)$,
$\Phi=|x|^4/16-|x|^2/2$ and $V=\tfrac14(|x|^2/4-1)^2|x|^2-\tfrac12(|x|^2-2)$ grows like
$|x|^6/64$: a strongly confining potential. Every term of~\eqref{eq:gauge} is linear in
$\varphi$, so Theorem~\ref{thm:hierarchy} gives
$\varphi(x_0,\tau_0)=M_r[\varphi]+S_r[-V\varphi]+S_r[-\nabla\cdot(\ft_1\varphi)]$, and
the last term can be written with values of $\ft_1\varphi$ only
(Proposition~\ref{prop:ibp}).

\subsection{Walks}

\paragraph{The exact (signed) walk.}
Formally, one may recurse through the Watson, potential, and transport
branches, dividing each contribution by its sampling probability. Such a
construction represents the exact identity, but unbiasedness of a fully
implemented walk requires a stopping rule, integrability, and justification
for exchanging the branching expectation with the recursion. In particular,
using an approximate terminal value at a positive stopping time introduces
bias. The direct signed construction is numerically unstable in the tests
below. The transport branch
carries a weight of size $O(r)$ with a random sign (the antithetic pair of
Proposition~\ref{prop:ibp}, which makes the term $O(r^2)$, cannot be kept along a single
path); multiplied over the $O(\tau_0/r^2)$ steps of a walk, such factors make the second
moment of the path weight grow, heuristically, like $\exp(C\tau_0/r)$. Experiment W1 below
shows weight standard deviations of $10^{12}$--$10^{26}$.

\paragraph{The moving heat-ball walk.}
In the spirit of the walk on moving spheres~\cite{DeaconuHerrmannMaire2017}, we let the heat
ball move with the transport. If the transport velocity $v$ and the potential $V$ are
constant on a heat ball, the step is exact without any sign: in the frame moving with velocity $v$ the equation
$\partial_\tau\varphi=\Delta\varphi-v\cdot\nabla\varphi-V\varphi$ becomes the heat equation for
$e^{V\tau}\varphi$, so Watson's identity gives
\begin{equation}\label{eq:moving-step}
  \varphi(x_0,\tau_0)=\E_{\mu_\kappa}\bigl[e^{-V\,r^2\sigma}\,
   \varphi\bigl(x_0+rz-v\,r^2\sigma,\ \tau_0-r^2\sigma\bigr)\bigr].
\end{equation}
For variable coefficients we freeze them over each heat ball, using a Heun average for
the displacement and the trapezoidal rule for the potential:
with $(z,\sigma)\sim\mu_\kappa$, $\Delta\tau=r^2\sigma$,
\begin{gather*}
 X^\ast=X+rz-v(X,\tau)\Delta\tau,\qquad
 X'=X+rz-\tfrac12\bigl(v(X,\tau)+v(X^\ast,\tau-\Delta\tau)\bigr)\Delta\tau,\\
 W'=W\exp\bigl(-\tfrac12\bigl(V(X,\tau)+V(X',\tau-\Delta\tau)\bigr)\Delta\tau\bigr),
\end{gather*}
and the radius at each step is $r=c\sqrt{4\pi\tau}$ with $c\in(0,1]$, so that the heat ball
stays in $\{\tau>0\}$ and the number of steps to reach $\tau=\varepsilon$ grows like
$\log(\tau_0/\varepsilon)/c^2$. The weights are positive and there is no branching; the price
is a bias from freezing the coefficients, which is $O(r^2)$ per unit time for smooth
coefficients (we verify the $c^2$ dependence numerically rather than prove it). Two
choices of $(v,V)$ are possible:
\begin{itemize}\itemsep0pt
\item \emph{direct}: $v=\ft$, $V=\nabla\cdot\ft$, for $u$ itself --- this is a heat-ball
discretization of the Feynman--Kac formula $u(x,\tau)=\E[u_0(Y_\tau)e^{-\int_0^\tau\nabla\cdot\ft(Y)}]$
with $\dd Y=-\ft(Y)\,\dd s+\sqrt2\,\dd B$;
\item \emph{gauge}: $v=\ft_1$, $V$ from~\eqref{eq:gauge}, for $\varphi$.
\end{itemize}
The two differ fundamentally for super-linear confining drifts. In the direct walk the
paths follow the backward characteristics $\dot Y=-\ft(Y)$, which for the double well,
$\dot Y=(|Y|^2/4-1)Y$, reach infinity in finite time; in the gauge walk the gradient part
never moves the paths, and instead the potential $V\sim|x|^6/64$ removes the weight of paths
that wander far. This is the main practical role of Proposition~\ref{prop:gauge}.

\paragraph{A corrected moving walk.}
The frozen-coefficient identity admits a formal exact correction. At the current point $x_k$ freeze $v=\ft(x_k)$ (or
$\ft_1(x_k)$ in the gauge variable) and $V_0=V(x_k)$, and write the remainders
$G=\ft-v$, $\delta V=V-V_0$. Then, without approximation,
\[
 \partial_\tau w=\Delta w-v\cdot\nabla w-V_0w-\nabla\cdot(Gw)-\delta V\,w .
\]
In the frame moving with $v$, and after the factor $e^{-V_0\Delta\tau}$, the first three terms form
the heat equation, so Theorem~\ref{thm:hierarchy} applies with the source
$-\nabla\cdot(Gw)-\delta V\,w$. Its divergence part is integrated by parts against $\psi_0$; since
$\psi_0$ vanishes on the boundary of the heat ball, \emph{no boundary term appears}, and the
result is a volume integral of $Gw$ against $\mu_{1/2,0,1}$ (Proposition~\ref{prop:ibp}),
which needs values only. On a heat ball of radius $r$, $G=O(L_f r)$ and $\delta V=O(L_Vr)$
($L_f,L_V$ local Lipschitz constants), so the two correction terms are $O(L_fr^2)$ and
$O(L_Vr^3)$, an order smaller than the transport term of the exact walk. Taking the branch
probabilities proportional to these sizes, $p_G=\min(p_{\max},a_GL_fr^2)$ and
$p_V=\min(p_{\max},a_VL_Vr^3)$ with $L_f,L_V$ estimated by finite differences at $x_k$, a correction
branch carries a weight of order one, and the second moment grows by $1+O(r^2)$ per step, i.e.\
by a bounded factor over the $O(1/r^2)$ steps of a walk under suitable
integrability and termination assumptions; a complete uniform second-moment
theorem is beyond the scope of this application. At the level of the recursive integral equation this correction removes
the freezing defect. A finite-stopping-time implementation is not exactly
unbiased when its terminal value is approximated, and the heuristic
second-moment count is not a convergence theorem. In the W1b calculations
the signed weights have high variance.

\paragraph{Terminal values and data.}
The walk stops when $\tau\le\varepsilon$. There $u(x,\tau)$ is replaced by the short-time
Gaussian approximation of the initial law, each mixture component $\mathcal N(m,S)$ being moved
by the drift linearized at $m$:
$m'=m+\ft(m)\tau$, $S'=S+2\tau I+\tau(BS+SB^{\!\top})$, $B=\nabla\ft(m)$ (local error $O(\tau^2)$).
When only data $y_1,\dots,y_n$ are available, $u_0$ is the empirical measure and the terminal
value is the Gaussian kernel sum $\frac1n\sum_i\mathcal N(x;\,y_i+\ft(y_i)\tau,\,2\tau I)$ over the
data points near $x$ (found with a $k$-d tree). The walk then estimates the density of the
\emph{empirical} measure transported by the forward process; since the paths themselves
spread over a distance $\sqrt{2\tau_0}$, no bandwidth has to be chosen.

\paragraph{Score.}
Branch choices, offsets and stopping times of the walk do not depend on $x_0$ (the radius
depends on $\tau$ only). For fixed random numbers the estimate $\hat u(x_0)$ is therefore a
smooth function of $x_0$, and $\nabla\hat u$ (computed by automatic differentiation along the
path) is an estimator of $\nabla u$ with the corresponding walk and
terminal-approximation biases; unbiasedness cannot be inferred solely
from pathwise differentiability. The score is estimated by
the ratio $\overline{\nabla\hat u}/\overline{\hat u}$ (bias $O(1/P)$ for $P$ paths), with standard
errors from independent batches.

\subsection{Heat-ball walks in two dimensions}\label{sec:walk-experiments}

The experiments ask three questions: whether the moving walk is exact up to its
frozen-coefficient bias (W1, W1b); whether the gauge transform keeps it stable under
a confining nonlinear drift, and how it compares with a kernel density estimate when
only data are available (W2); and what the two parts of a learned drift do (W3).
All walks use $P=10^5$--$2\times10^5$ paths per anchor, split into $10$ batches for
standard errors, and antithetic pairs $(z,-z)$ at every step. Scripts
\texttt{w1\_walk\_linear.py}--\texttt{w3\_drift\_parts.py}.

\paragraph{W1. Linear drift with rotation: exactness and the choice of radius.}
The drift is $\ft=-\mathrm{diag}(0.5,2)\,x+2Jx$ (anisotropic, with rotation; then
$\ft_1\cdot\nabla\Phi\ne0$), the initial law is \texttt{asymmetric\_gmm}, and
$u$, $\nabla u$ and the score at $\tau_0=0.5$ are in closed form. The walk is stopped at
$\tau\le10^{-3}$ with the exact $u$, so that only the walk itself is tested. Table~\ref{tab:w1}
reports the error at $8$ typical anchors (drawn from $p_{\tau_0}$) and $4$ low-density anchors
(below the $10\%$ quantile of $p_{\tau_0}$), normalized also by the Monte Carlo standard error
($z$). With $c=1$ the frozen-coefficient bias is visible ($|z|\approx3$--$4$, $4$--$8\%$); with
$c\le0.5$ the error is at the Monte Carlo level ($|z|\approx1$), at about five times the number
of steps. The gauge walk gives the most accurate score ($2$--$4\%$ relative error versus
$7$--$14\%$ for the direct walk at the same cost). The exact signed walk, by contrast, is
useless: with $2\times10^4$ paths its weights have standard deviation $10^{12}$--$10^{26}$, half
of them are negative, and its relative standard error is $1$--$100$. Figure~\ref{fig:w1} (left) shows the errors as functions of $c$.

\begin{table}[ht]
\centering
\caption{W1: moving heat-ball walk for a linear drift with rotation, $\tau_0=0.5$,
$2\times10^5$ paths per anchor. Mean over anchors of $|\hat u/u-1|$, of $|z|=|\hat u-u|/\mathrm{s.e.}$,
and of the relative score error $|\hat s-s|/|s|$; mean number of steps and seconds per anchor
(2-core CPU, including the score). $^\dagger$Two anchors of each kind.}
\label{tab:w1}
\small
\begin{tabular}{llcccccc}
\toprule
walk & $c$ & \multicolumn{2}{c}{$|\hat u/u-1|$ (typical / low)} & $|z|$ & score error (typical / low) & steps & s \\
\midrule
direct & 1    & 0.040 & 0.085 & 3.6 & 0.118 / 0.176 & 42  & 11 \\
direct & 0.5  & 0.012 & 0.029 & 0.8 & 0.073 / 0.137 & 193 & 31 \\
direct & 0.25$^\dagger$ & 0.051 & 0.024 & 1.4 & 0.162 / 0.061 & 789 & 93 \\
gauge  & 0.5  & 0.007 & 0.023 & 1.3 & 0.038 / 0.022 & 193 & 54 \\
\midrule
exact signed walk & 0.6--1 & \multicolumn{2}{c}{relative s.e.\ $1$--$109$} & --- & --- & 45--139 & \\
\bottomrule
\end{tabular}
\end{table}

\begin{figure}[ht]
\centering
\includegraphics[width=\textwidth]{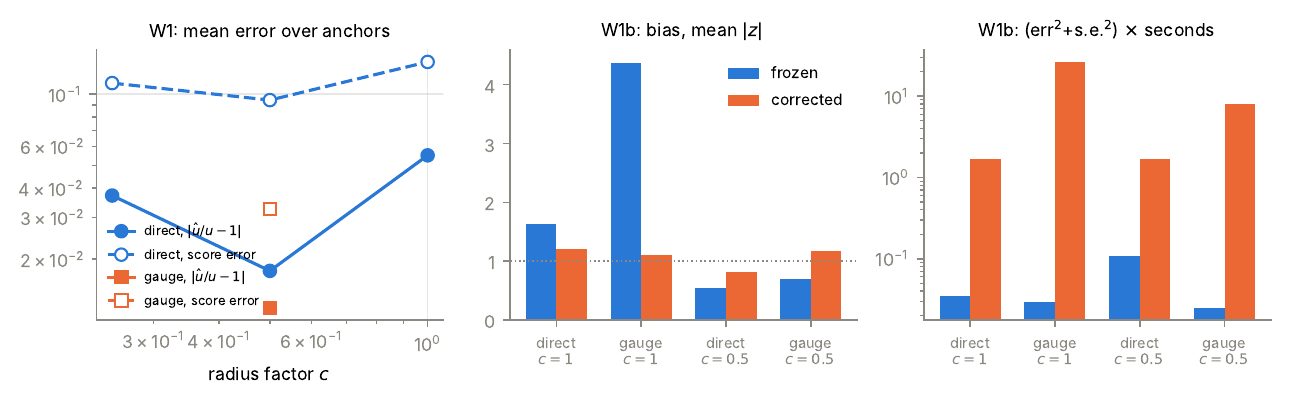}
\caption{Left: W1, mean error over the anchors of the density (solid) and of the score (dashed)
against the radius factor $c$ (the point $c=0.25$ uses four anchors). Middle and right: W1b,
frozen versus corrected walk on four anchors --- mean $|z|$ (dotted: $1$) and
work-normalized error. The middle and right panels show the increased work-normalized error associated with removing the bias, not an efficiency gain.}
\label{fig:w1}
\end{figure}

\paragraph{W1b. The corrected moving walk.}
On the W1 problem (two typical and two low-density anchors, $10^5$ paths, terminal exact at
$\tau\le10^{-3}$) and on the double well of W2 (gauge scheme, $c=0.5$, two anchors),
Table~\ref{tab:w1b} and Figure~\ref{fig:w1} compare the frozen walk with a walk using
integration-by-parts corrections ($a_G=0.02$, $a_V=0.005$,
$p_{\max}=0.3$). The corrected estimates are compatible with the
reference within their reported sampling errors (mean $|z|=0.8$--$1.2$,
maximum $2.4$), including the $c=1$ gauge case where the frozen estimate
shows a larger standardized discrepancy (mean $|z|=4.4$, maximum $10$).
This does not mean that the corrected estimates have smaller observed
absolute error: the much larger sampling variance can dominate. But the signed correction weights raise the
standard error by a factor $2$--$40$ and the cost per path by a factor three, so that at equal
work its mean squared error is $15$--$900$ times larger than that of the frozen walk with the
same $c$.
For computing densities and scores the frozen walk with $c\le0.5$ is therefore the method of
choice, and the corrected walk provides a numerical check on the freezing bias,
subject to the chosen terminal approximation and sampling error. Keeping both
antithetic points of each correction (a branching walk) would reduce the correction weights by
another factor $r$; we have not tried it.

\begin{table}[ht]
\centering
\caption{W1b: frozen versus corrected moving walk on the W1 problem; means over
four anchors of $|\hat u/u-1|$, of the relative standard error and of $|z|$; median relative
score error; seconds per anchor (with the score); work-normalized error
$(\mathrm{err}^2+\mathrm{s.e.}^2)\times$seconds. Last rows: double well (W2), gauge, $c=0.5$.
Script \texttt{w1b\_walk\_corrected.py}.}
\label{tab:w1b}
\small
\begin{tabular}{llcccccc}
\toprule
scheme, $c$ & walk & $|\hat u/u-1|$ & s.e. & $|z|$ (max) & score error & s & work-norm.\ error \\
\midrule
direct, 1   & frozen    & 0.057 & 0.044 & 1.6 (3.3)  & 0.17 & 5  & $3.5\times10^{-2}$ \\
            & corrected & 0.145 & 0.210 & 1.2 (2.4)  & 0.80 & 17 & $1.7$ \\
gauge, 1    & frozen    & 0.053 & 0.019 & 4.4 (10.3) & 0.06 & 8  & $2.9\times10^{-2}$ \\
            & corrected & 0.400 & 0.732 & 1.1 (2.2)  & 1.20 & 24 & $26$ \\
direct, 0.5 & frozen    & 0.027 & 0.050 & 0.6 (1.4)  & 0.23 & 16 & $0.11$ \\
            & corrected & 0.093 & 0.114 & 0.8 (1.6)  & 0.21 & 52 & $1.7$ \\
gauge, 0.5  & frozen    & 0.011 & 0.021 & 0.7 (1.6)  & 0.07 & 26 & $2.5\times10^{-2}$ \\
            & corrected & 0.180 & 0.200 & 1.2 (2.3)  & 0.63 & 75 & $8.0$ \\
\midrule
double well, gauge, 0.5 & frozen    & 0.013 & 0.018 & 0.7 (1.1) & 0.07 & 34 & --- \\
                        & corrected & 0.023 & 0.049 & 0.6 (0.9) & 0.23 & 82 & --- \\
\bottomrule
\end{tabular}
\end{table}

\paragraph{W2. Nonlinear drift, and data instead of a density.}
For the double well with rotation of E5(b) (initial law \texttt{asymmetric\_gmm}) the reference is
the $400^2$ finite-volume solution. We evaluate $p_{\tau_0}$ and the score at $4$ typical and $2$
low-density anchors for $\tau_0\in\{0.2,0.5\}$ ($c=0.5$, $5\times10^4$ paths), with two kinds of
terminal value: the short-time Gaussian approximation of the known initial mixture
($\varepsilon=0.01$), and only data, $n\in\{10^3,10^4,10^5\}$ samples of $p_0$
($\varepsilon=10^{-3}$, kernel sum over the $1024$ nearest points; for $n=10^5$ each batch of
paths uses a fresh random subset of $10^4$ points, which is unbiased for the full kernel sum and
keeps the cost independent of $n$; with all $10^5$ points and a fixed $k=1024$ the kernel sum is
truncated and $p$ was underestimated by $37\%$ at the one anchor we tried). Table~\ref{tab:w2} summarizes. Three findings:
(i) the direct walk is unusable for this drift --- at $\tau_0=0.5$ every run had paths that
reached infinity along the backward characteristics, and at $\tau_0=0.2$ half of them --- while
the gauge walk never diverged; (ii) with the known initial law the gauge walk gives $p$ to
$1$--$2\%$ at typical anchors and $5\%$ at low-density anchors (about two standard errors, so a
small bias from the frozen coefficients and the terminal approximation is visible there) and
the score to $2$--$5\%$; (iii) with data only, the error drops from $n=10^3$ to $n=10^4$ and at $n=10^5$
is within a factor two of the known-law walk, whereas the global alternative --- simulating the
$n$ data points to $\tau_0$ and evaluating a Gaussian kernel density estimate with Scott's
bandwidth~\cite{Scott1992} --- is biased by its bandwidth: $5$--$9\%$ at typical anchors, $13$--$24\%$ at
low-density anchors and $10$--$17\%$ score error even with $n=10^5$. The kernel estimate is,
however, much cheaper in two dimensions (seconds in total versus about a minute per anchor
for the walk with the score): the advantage of the walk is the absence of a smoothing
bandwidth, not cost.

\begin{table}[ht]
\centering
\caption{W2: double well with rotation; mean over anchors of $|\hat p/p-1|$ and of the relative
score error, typical / low-density anchors. ``walk'': gauge walk, $5\times10^4$ paths;
``KDE'': forward simulation of the $n$ data points plus Gaussian kernel density estimate.
Script \texttt{w2\_walk\_nonlinear.py}.}
\label{tab:w2}
\small
\begin{tabular}{llcccc}
\toprule
 & & \multicolumn{2}{c}{$\tau_0=0.2$} & \multicolumn{2}{c}{$\tau_0=0.5$} \\
method & terminal / data & $|\hat p/p-1|$ & score error & $|\hat p/p-1|$ & score error \\
\midrule
direct walk & known $p_0$ & \multicolumn{2}{c}{diverged in 3 of 6 anchors} & \multicolumn{2}{c}{diverged in 6 of 6} \\
gauge walk & known $p_0$ & 0.007 / 0.051 & 0.023 / 0.049 & 0.018 / 0.047 & 0.054 / 0.047 \\
gauge walk & $n=10^3$ & 0.057 / 0.101 & 0.073 / 0.170 & 0.050 / 0.116 & 0.133 / 0.159 \\
gauge walk & $n=10^4$ & 0.010 / 0.046 & 0.039 / 0.065 & 0.042 / 0.058 & 0.064 / 0.092 \\
gauge walk & $n=10^5$ (subsets of $10^4$) & 0.004 / 0.042 & 0.019 / 0.070 & 0.030 / 0.062 & 0.055 / 0.083 \\
\midrule
KDE & $n=10^3$ & 0.31 / 1.12 & 0.51 / 0.63 & 0.31 / 0.56 & 0.48 / 0.47 \\
KDE & $n=10^4$ & 0.23 / 0.68 & 0.36 / 0.42 & 0.18 / 0.33 & 0.29 / 0.30 \\
KDE & $n=10^5$ & 0.09 / 0.24 & 0.17 / 0.10 & 0.05 / 0.13 & 0.15 / 0.13 \\
\bottomrule
\end{tabular}
\end{table}

\paragraph{W3. The two parts of a learned drift.}
The forward model of the authors' earlier preprint (arXiv:2606.27954), $\ft=\nabla\phi_\theta+\nabla^\perp\psi_\theta$ with networks
$\phi_\theta,\psi_\theta$ (width $64$, two residual blocks), was trained on \texttt{asymmetric\_gmm}
with the corrected forward-training code (MMD to $\mathcal N(0,I)$ at $\tau=1$ and the
$\tau=0$ / $\tau=1$ boundary losses, $600$ epochs). By Proposition~\ref{prop:gauge} its
curl-free part enters only through $\Phi=-\phi_\theta$ and the potential $V$, and its
divergence-free part only as transport. We solved the Fokker--Planck equation
(finite volumes, $160^2$ cells, drift frozen on time segments of $0.025$) for the full drift
and with either part removed (Figure~\ref{fig:w3}, Table~\ref{tab:w3}). The learned drift is
almost purely curl-free: on $[-6,6]^2$ the mean of $|\ft_1|$ is $0.039$ against $2.2$
for the gradient part, removing the rotation changes the density by $0.01$--$0.02$ in $L^1$,
and both the full and the curl-free drift bring the law to within $0.10$ of $\mathcal N(0,I)$ at
$\tau=1$, whereas the divergence-free part alone leaves it at $1.19$. In this run, training the
drift towards a Gaussian terminal law produced no appreciable rotation; to study the effect
of a rotational component it has to be imposed or rewarded explicitly. For comparison, the double well with rotation
($|\ft_1|\le1.5$) shows how a substantial divergence-free part redistributes mass along
the level sets of $\Phi$ without changing the potential landscape: removing it changes the
density by $0.24$--$0.32$ in $L^1$. The potential $V$ makes the role of the curl-free part
visible: it is negative (creation) only in a bounded region near the origin --- for the double
well, the ring $|x|\approx2$ that contains the wells --- and grows rapidly outside, where it
absorbs. Finally, the gauge walk applied to the learned drift (with $V$, $\Phi$ and
$\ft_1$ tabulated on the grid) reproduces the grid density at six anchors at
$\tau=0.5$ to within $2\%$ ($|z|\le1.6$, $10^5$ paths).

\begin{table}[ht]
\centering
\caption{W3: $L^1$ distance between the densities obtained with the full drift and with one
part removed, and distance to $\mathcal N(0,I)$ at $\tau=1$. Script \texttt{w3\_drift\_parts.py}.}
\label{tab:w3}
\small
\begin{tabular}{lccccc}
\toprule
drift & & $\tau=0.25$ & $\tau=0.5$ & $\tau=1$ & $\|p_1-\mathcal N(0,I)\|_{L^1}$ \\
\midrule
learned & full vs.\ curl-free only & 0.010 & 0.012 & 0.020 & full $0.102$ \\
        & full vs.\ div-free only  & 0.70  & 0.96  & 1.21  & curl-free only $0.102$ \\
        &                          &       &       &       & div-free only $1.19$ \\
\midrule
double well + rotation & full vs.\ curl-free only & 0.32 & 0.31 & 0.24 & --- \\
        & full vs.\ div-free only & 0.46 & 0.63 & 0.81 & --- \\
\bottomrule
\end{tabular}
\end{table}

\begin{figure}[ht]
\centering
\includegraphics[width=\textwidth]{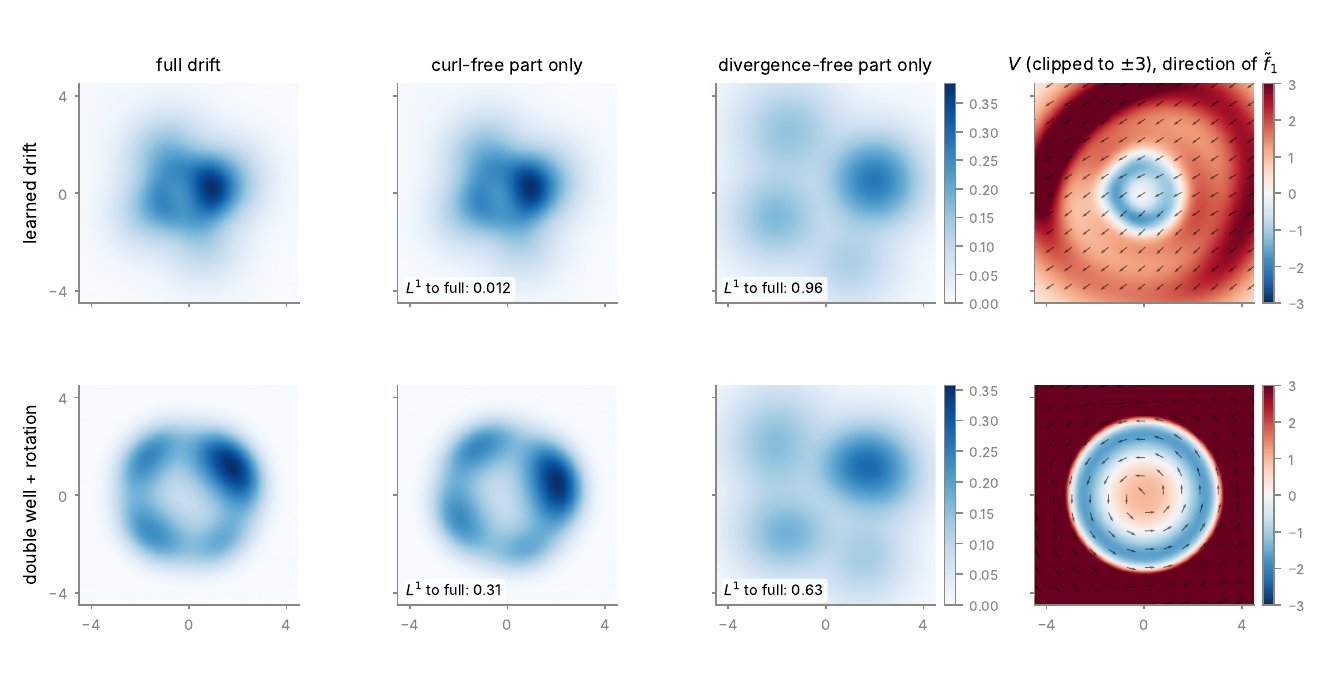}
\caption{W3: densities at $\tau=0.5$ (square root of the density shown) under the full drift
and with one part removed, and the gauge potential $V$ (clipped to $\pm3$) with the direction of
$\ft_1$. Top: drift learned by the forward-training code; bottom: double well with
rotation. The three densities in a row share one colour scale; the insets give the $L^1$
distance to the density with the full drift (Table~\ref{tab:w3}). Plotted window $[-4.5,4.5]^2$.}
\label{fig:w3}
\end{figure}

The two-dimensional examples show that the gauge reformulation can give
stable pointwise evaluations in cases where the direct walk is numerically
unstable, and that it avoids choosing a kernel-smoothing bandwidth. They
also show the variance cost of signed corrections and do not establish
a uniform efficiency or convergence advantage, particularly in higher
dimension. The walks are included to illustrate a possible recursive use
of the density identity rather than as a second main contribution.

%==========================================================
\section{Conclusion}\label{sec:conclusion}
%==========================================================

The questions motivating this work arise in diffusion-based generative
models with general drift: how the gradient and divergence-free parts
enter the density evolution, how much can be expressed through local
information, and how to represent low-density regions without relying
solely on an unbounded log-density. We have studied these questions at
the level of local Fokker--Planck identities and their numerical evaluation.
In particular, we obtained heat-ball relations for $u$, $\log u$,
$u\log u$, and $\nabla\log u$ with the source terms retained. The
associated integrals share a Beta--Gamma family of measures that permits
exact sampling, quadrature, and explicit error estimates.

The experiments confirm several consequences of the analysis against
exact or independently computed references, and they map out its limits.
Log-density and score identities can be evaluated accurately in cases
where direct density estimation suffers from large relative variance.
Gaussian-semigroup estimators have similar leading-order variance at
matched time lag, and angular integration remains costly in high
dimension. The derivative-reduced residuals provide consistency
diagnostics, but their Monte Carlo noise and evaluation costs depend on
the chosen implementation. The two-dimensional walk experiments illustrate
a recursive alternative, without a general stability or convergence claim.

The entropy density speaks to the low-density motivation because
$|u\log u|/u=|\log u|$ grows as $u\to0^+$, while $u\log u$ itself
vanishes. This pointwise property motivates entropy-aware objectives,
though it proves nothing about stronger score learning or better generated
samples. Likewise, a heat ball is local in the diffusion state and
backward time, but turning this into spatially local dependence among
pixels or physical degrees of freedom would require additional structure.
What the two drift components do in reverse generation, and how to design
entropy-based training on top of the present identities, are therefore
separate questions. What this paper provides is a common local
representation for studying them, together with a way of integrating it
with controlled sampling error.

\paragraph{Code availability.}
The samplers, estimators, quadrature rules, reference solvers, and
experiment scripts (the \texttt{hbmc} library and \texttt{experiments/}
directory) are provided as supplementary material and are intended to
be made public with the published article.

%==========================================================

%==========================================================
\appendix

\section{Proof of Theorem~\ref{thm:local}}\label{app:proof}

Fix $T>0$, let $A=A(T)$, $C=C(T)$, $K=A+C$, and for $0\le a<b\le T$ work in
$X_{a,b}=C([a,b];\BUC^1(\R^d))$ with
$\|u\|_{X_{a,b}}=\sup_{[a,b]}(\|u\|_{\Linf}+\|\nabla u\|_{\Linf})$; the heat
semigroup is strongly continuous on $\BUC^1$.

\emph{Step 1 (one contraction step).} A mild solution on $[a,b]$ with data
$u(a)$ is a fixed point of
$\Phi(u)(\tau)=e^{(\tau-a)\Delta}u(a)-\int_a^\tau e^{(\tau-\ell)\Delta}N[u](\ell)\,\dd\ell$.
Since $\|N[u]\|_{\Linf}\le K\|u\|_{X_{a,b}}$,
$\|e^{t\Delta}g\|_{\Linf}\le\|g\|_{\Linf}$ and
$\|\nabla e^{t\Delta}g\|_{\Linf}\le C_dt^{-1/2}\|g\|_{\Linf}$ (the $L^1$ norm of
$|\nabla\Phi_t|$ is $C_dt^{-1/2}$),
$\|\Phi(u)-\Phi(\tilde u)\|_{X_{a,b}}\le K\sqrt h(\sqrt h+2C_d)\|u-\tilde u\|_{X_{a,b}}$
with $h=b-a$. For $h\le T_0(\theta):=(\sqrt{C_d^2+\theta/K}-C_d)^2$ the
affine map $\Phi$ is a $\theta$-contraction and has a unique fixed point.

\emph{Step 2 (global existence).} $T_0(\theta)$ does not depend on $u(a)$, so
the step can be repeated from $a=0,T_0,2T_0,\dots$ and covers $[0,T]$;
$T$ is arbitrary. For bounded drift, mass conservation follows by applying spatial cutoff
functions to the divergence-form equation and passing to the limit (or,
equivalently, from the conservative diffusion semigroup); positivity follows
from the parabolic maximum principle and the positive fundamental solution.

\emph{Step 3 (regularity and bounds).} Interior Schauder estimates give the
regularity in~(a). The function $M_0e^{C(\tau)\tau}$ is a supersolution
because $C\ge-\nabla\cdot\ft$, which gives the upper bound by comparison.
By Aronson's theorem~\cite{Aronson1968}, the fundamental solution $\Gamma$ of
the divergence-form operator in~\eqref{eq:heat}, whose lower-order
coefficient $\ft$ is bounded on $[0,T]$, satisfies
$\Gamma(x,\tau;y,0)\ge c_T\tau^{-d/2}e^{-\gamma_T|x-y|^2/\tau}$; integrating
against $u_0$ gives the first inequality in~\eqref{eq:aronson-lower}, and
$|x-y|^2\le2|x|^2+2|y|^2$ gives the second.

\section{Proof of Lemma~\ref{lem:monotone}}\label{app:proof-lemmas}

Under $x-x_0=rz$, $\tau_0-\tau=r^2\sigma$, write
$J(r)=A(r)+B(r)$ with
$A(r)=\iint_{\Ecal_1}F(x_0+rz,\tau_0-r^2\sigma)|z|^2/\sigma^2\,\dd z\,\dd\sigma$
and $B(r)=\int_0^r4d\rho^{-d-1}\iint_{\Ecal(\rho)}h\,\psi$. Differentiating
$A$, returning to $(x,\tau)$ coordinates and integrating by parts as in
\cite[\S2.3.2]{Evans2010} gives
$A'(r)=-4dr^{-d-1}\iint_{\Ecal(r)}(\partial_\tau F-\Delta F)\psi
=-4dr^{-d-1}\iint_{\Ecal(r)}h\,\psi$, while
$B'(r)=4dr^{-d-1}\iint_{\Ecal(r)}h\,\psi$; hence $J'\equiv0$. As $r\to0$,
$A(r)\to4F(x_0,\tau_0)$ because $\iint_{\Ecal_1}|z|^2/\sigma^2=4$, and
$B(r)\to0$. Rescaling $B(r)$ to $\Ecal_1$ gives $4S_r[h]$, so
$4F(x_0,\tau_0)=4M_r[F]+4S_r[h]$.

\end{document}